\documentclass[lettersize,journal,doublecolumn]{IEEEtran}
\IEEEoverridecommandlockouts
\usepackage{amsmath,amsfonts}
\usepackage{amssymb}
\usepackage[caption=false,font=footnotesize,labelfont=sf,textfont=sf]{subfig}
\usepackage{stfloats}
\usepackage{url}
\usepackage{graphicx}
\usepackage{cite}
\usepackage{float}
\usepackage{subfig}
\usepackage{multirow}
\usepackage{color}
\usepackage{makecell}
\usepackage{booktabs}
\usepackage{mathtools}
\usepackage{bm}
\usepackage[colorlinks,citecolor=blue,linkcolor=blue]{hyperref}
\newtheorem{corollary}{Corollary}

\newtheorem{proposition}{Proposition}
\newtheorem{theorem}{Theorem}
\newtheorem{assumption}{Assumption}

\newtheorem{algorithm}{Algorithm}
\usepackage{algorithm, algorithmicx, algpseudocode}
\begin{document}

\title{Generator-Guided Inverse Sampling for L\'evy-Driven Generative Models}

\author{Tianfu Qi, \emph{Graduate Student Member, IEEE}, Jun Wang, \emph{Senior Member, IEEE}, Jun Zhang, \emph{Fellow, IEEE}
\thanks{Tianfu Qi, Jun Wang are with the National Key Laboratory of Wireless Communications, University of Electronic Science and Technology of China, Chengdu 611731, China (e-mail: 202311220634@std.uestc.edu.cn; junwang@uestc.edu.cn).}
\thanks{J. Zhang is with the Department of Electronic and Computer Engineering, Hong Kong University of Science and Technology, Hong Kong, China (e-mail: eejzhang@ust.hk).}
}

\maketitle

\begin{abstract}
This paper studies inverse sampling for L\'evy-driven generative models from the perspective of Markov generators. Unlike conventional diffusion models, L\'evy-driven dynamics involve infinite jump activities, which makes their reverse process nonlocal and difficult to characterize using score information alone. We address this challenge by analyzing the forward and reversed generators. It is derived that the reversed jump component generally becomes a state-dependent Markov jump process governed by a nonlocal density ratio. This observation motivates a structured reverse sampler that decomposes the dynamics into diffusion, small jump, and large jump components. Based on this characterization, we develop a computationally tractable sampler for a class of isotropic linear L\'evy SDEs with symmetric $\alpha$-stable jump components. For the jump component, the neural network is used only to amortize the rate of large jump activities, while jump amplitudes are generated from analytically derived conditional distributions, which improves interpretability and controllability. Efficient implementation techniques are further introduced under this setting to avoid expensive high-dimensional integration and sampling. The sampler is further adapted to approximate observation-guided sampling and applied to OFDM-SISO channel estimation under mixed Gaussian and impulsive noise. Simulations show robust estimation performance with a favorable tradeoff between complexity and performance.
\end{abstract}


\section{Introduction}
Diffusion models based on score matching have emerged as a powerful class of generative models \cite{paper1,paper2,paper3,paper4}. They have also attracted increasing attention in signal processing applications, including inverse problem solving \cite{paper5,paper6,paper7,paper8,paper9,paper10}, image restoration \cite{paper11,paper12,paper13,paper14}, and wireless channel estimation \cite{paper15,paper16,paper17}. By gradually perturbing data with Gaussian noise and learning the reverse dynamics through score matching, these models provide a stochastic framework for sampling from complex data distributions.

Despite their success, conventional diffusion models are mainly built on the Wiener process. Their reverse process is characterized by local score information and infinitesimal Gaussian perturbations \cite{paper1,paper3}. As a result, long-range probability transport is usually achieved through a sequence of small reverse steps. This local update mechanism may affect sampling efficiency and modeling flexibility \cite{paper18}, especially when the target distribution has heavy tails, impulsive components, or highly multimodal structures \cite{paper19,paper20}.

Lévy-driven generative models provide a natural extension by incorporating jump components into the forward dynamics \cite{book1}. The resulting nonlocal transitions are well suited for modeling heavy-tailed perturbations. They can also provide more flexible distribution transport than standard diffusion dynamics \cite{paper19}. However, these advantages bring new challenges for inverse sampling. Unlike the case without jumps, the time reversal of a Lévy-driven process cannot be fully characterized by local score information alone \cite{paper19,paper20}. Its reversed jump component depends on a nonlocal density ratio. It generally becomes a Markov jump process with state-dependent kernels, which makes the inverse sampler difficult to design.

It has been shown in \cite{paper19} that the denoising score matching (DSM) framework can still be applied to inverse sampling for Lévy processes. Under this formulation, the neural network is required to learn denoising targets associated with impulsive perturbations, such as those induced by $\alpha$-stable distributions. However, different from white Gaussian noise (WGN), impulse noise (IN) contains multiple outliers with large amplitudes. Compared with WGN, which has a relatively stable envelope over a long time interval, IN is much more impulsive and chaotic. Therefore, it is difficult to learn the statistical characteristics of impulse noise at different moments. In this case, full neural parameterization makes the sampler less interpretable and controllable.

Motivated by these observations, we aim to establish a generator-guided inverse sampling approach for L\'evy-driven generative models. The generator provides a fundamental characterization of the infinitesimal statistics of Markov processes and naturally reveals the nonlocal structure of the reversed jump dynamics. Based on this structural characterization, we design a practical sampler under isotropic linear $\alpha$-stable assumptions, so that the main components of the reverse dynamics can be implemented in a statistically interpretable and computationally tractable manner. The main contributions of this paper are summarized as follows.
\begin{itemize}
    \item \textbf{First, we provide a generator-based characterization of the time-reversed dynamics of L\'evy-driven Markov processes.} The analysis shows that the reversed jump component is generally a state-dependent Markov jump process whose kernel involves a nonlocal density ratio. This result clarifies why reverse sampling for L\'evy-driven processes cannot, in general, be reduced to score-based local diffusion updates.

    \item \textbf{Second, motivated by this characterization, we develop a practical inverse sampler for isotropic linear L\'evy SDEs with symmetric $\alpha$-stable jump components.} The reverse dynamics are decomposed into diffusion, small jump, and large jump parts. The small jump contribution is approximated through a Gaussian surrogate, while large jumps are handled explicitly through a rate-and-amplitude decomposition. The neural network is used to amortize the large jump rate, whereas jump amplitudes are sampled from analytically derived conditional distributions.

    \item \textbf{Third, we adapt the proposed sampler to approximate observation-guided sampling under observation models and demonstrate its use in OFDM-SISO channel estimation with mixed Gaussian and impulsive noise.}  In particular, the prior-trained jump rate network is retained as a proposal mechanism and the observation likelihood is incorporated through reweighted distributions. The resulting algorithm provides a robust estimator under heavy-tailed perturbations and offers an interpretable alternative to fully neural reverse-transition parameterizations.
\end{itemize}

The remainder of this paper is organized as follows. In Section \ref{section_1}, we describe the Lévy process considered throughout this paper and present the corresponding assumptions. Statistical analyses from the generator perspective are given in Section \ref{section_2}. Based on the above theoretical results, an efficient inverse sampling algorithm is designed in Section \ref{section_3}. Section \ref{section_4} provides simulation results for channel estimation under mixed channel noise to validate the proposed inverse sampling framework. Section \ref{section_5} concludes the paper.

\emph{Notations:} Vectors are denoted by bold letters and are assumed to be column vectors by default. Superscripts and subscripts are also used to denote vectors and vector sequences. For example, given an $m$-dimensional vector $\bm{x}$, we have $\bm{x}=[x_1,x_2,\cdots,x_m]^\top$ and $\bm{x}_1^p=[\bm{x}_1,\cdots,\bm{x}_p]^\top$. Uppercase letters denote random variables (RVs), and lowercase letters denote their realizations, e.g., $X$ and $x$. The operators $\Delta$, $\nabla$, and $\nabla\cdot$ denote the Laplace operator, gradient operator, and divergence operator, respectively. The notation $\text{diag}(a_1,a_2,\cdots,a_L)$ represents a diagonal matrix with diagonal elements $a_1,a_2,\cdots,a_L$. $\langle\bm{x},\bm{y}\rangle$ denotes the inner product of vectors $\bm{x}$ and $\bm{y}$. $1_{\mathcal{A}}(x)$ denotes the indicator function, which equals 1 if $x\in\mathcal{A}$. The real number domain is denoted by $\mathbb{R}$. The notation $X\sim\cdot$ indicates that $X$ follows a certain distribution.

\section{Preliminaries}\label{section_1}
In this paper, we consider a Lévy process with drift and diffusion components, which can be generally written as
\begin{align}\label{SDE}
d{\bm{X}_t} = \bm{b}(\bm{X}_t,t)dt + \Phi_G(t) d{\bm{W}_t} &+\Phi_S(t)d\bm{L}_t,\nonumber\\
&t:0\rightarrow T,\bm{X}_t\in\mathbb{R}^D
\end{align}
where $D$ denotes the dimension of the SDE. The term $\bm{b}(\bm{X}_t,t)$ is the drift coefficient of the dynamics, which depends on both the time index $t$ and the system state $\bm{X}_t$. The processes $\bm{W}_t$ and $\bm{L}_t$ denote the Wiener process and the $\alpha$-stable process, respectively. The matrices $\Phi_G(t)$ and $\Phi_S(t)$ denote the scaling matrices for $\bm{W}_t$ and $\bm{L}_t$, respectively. For isotropic cases, $\Phi_G(t)$ and $\Phi_S(t)$ reduce to scaled identity matrices. In the sequel, we focus on the SDE under the following assumptions.

\begin{assumption}\label{assumption_1}
The Wiener process $\bm{W}_t$ and the Lévy process $\bm{L}_t$ are mutually independent.
\end{assumption}
\begin{assumption}\label{assumption_2}
$\bm{L}_t$ is a $D$-dimensional symmetric isotropic $\alpha$-stable Lévy process with stability index $\alpha\in(0,2)$.
\end{assumption}

Our ultimate goal is to investigate the characteristics of the reverse process corresponding to the SDE in \eqref{SDE} and design an efficient reverse sampler. Unfortunately, unlike a diffusion process driven only by Brownian motion, \eqref{SDE} is a Lévy process with infinite jump activities. Its reverse process cannot be described using only local information. In other words, if we set $\Phi_S(t)\equiv \bm{0}$, the backward process of \eqref{SDE} can be determined by the score information, i.e., $\nabla\log p(\bm{x}_t)$. This information is local since it only requires the gradient of the density. This property can be obtained by analyzing the forward and backward Kolmogorov equations, which are special cases of the Kramers-Moyal (KM) expansion. The core idea of the KM expansion is to express the time evolution of the probability density as an infinite series of state-space changes.

For diffusion processes driven by Brownian motion, this series can be truncated after the first two terms, which leads to the forward Kolmogorov equation. However, if $\Phi_S(t)\neq\bm{0}$, large jumps may appear in \eqref{SDE} due to $d\bm{L}_t$. As a result, $\Vert\bm{X}_{t+\Delta t}-\bm{X}_t\Vert$ can be quite large. Moreover, the KM expansion involves the $j$-th moment of $X_t$, where $j\in\mathbb{Z}^+\bigcup\{0\}$. For a Lévy process driven by an $\alpha$-stable process, finite $p$-th moments with $p\geq\alpha$ do not exist. Therefore, the reverse process becomes much more complicated. In the following, we use the generator to analyze the statistical properties of \eqref{SDE}.

Given the forward dynamics in \eqref{SDE}, its generator is given by \cite{book1}
\begin{align}\label{forward_generator}
\left( {\mathcal{L}_Ff} \right)(\bm{x}) =& \bm{b}(\bm{x},t)\cdot\nabla f(\bm{x})+\frac{1}{2}{\rm{Tr}}\left( {\Sigma (t)\nabla^2 f(\bm{x})} \right)\nonumber\\
&+ (\mathcal{L}_{J,F}f)(\bm{x})
\end{align}
where $\Sigma(t)=\Phi_G(t)\Phi_G(t)^\top$ and $(\mathcal{L}_{J,F}f)(\bm{x})$ denotes the forward generator of the process $\Phi_S(t) d\bm{L}_t$. The backward generator of the reversed version of \eqref{SDE} will be characterized in the next section. Before ending this section, we introduce the following assumptions, which ensure the validity of the subsequent analysis.

\begin{assumption}\label{assumption_3}
The forward SDE admits a non-explosive Markov solution with transition density $p(\bm{x}_t)$ with respect to the Lebesgue measure.
\end{assumption}

\begin{assumption}\label{assumption_4}
The density $p(\bm{x}_t)$ is strictly positive and sufficiently smooth so that $\nabla\log p(\bm{x}_t)$ is well defined.
\end{assumption}

\section{Generator-Based Reverse Dynamics}\label{section_2}
In this section, we first derive the backward generator corresponding to the forward process in \eqref{SDE}. Then, we theoretically analyze how to decompose the whole inverse sampling process and generate each component.

\subsection{Backward generator}
Throughout this paper, the term ``backward generator'' refers to the infinitesimal generator of the time-reversed process. It should not be confused with the backward Kolmogorov operator associated with the forward process. For notational clarity, we write $p(\bm{x}_t)$ for the marginal density of $\bm{X}_t$.

Based on \eqref{SDE} and its forward generator \eqref{forward_generator}, the backward generator is given in the following proposition.

\begin{proposition}[Backward generator]\label{Proposition_1}
\textit{Consider the forward SDE in \eqref{SDE} with L\'evy measure $\nu$. Let $\nu_t = (\Phi_S(t))_{\#}\nu$ be the push-forward L\'evy measure induced by the jump scaling matrix $\Phi_S(t)$. Suppose that the forward process admits a strictly positive marginal density $p(\bm{x}_t)$, and that $p(\bm{x}_t)$ and $\bm{b}(\bm{x},t)$ are sufficiently smooth and integrable so that the following integrations by parts and principal-value integrals are well defined. If the scaled L\'evy measure $\nu_t$ is symmetric, then for any test function $f\in C_c^2(\mathbb{R}^D)$, the generator of the reversed process can be written as
\begin{align}\label{proposition_equation_1}
(L_B f)(\bm{x}_t)
=&
\left[-\bm{b}(\bm{x}_t,t)+\Sigma(t)\nabla \log p(\bm{x}_t)\right]\cdot \nabla f(\bm{x}_t)
\nonumber\\
&+\frac{1}{2}\operatorname{Tr}\left(\Sigma(t)\nabla^2 f(\bm{x}_t)\right)+(L_{J,B} f)(\bm{x}_t)
\end{align}
\begin{align}\label{proposition_equation_2}
(L_{J,B} f)(\bm{x}_t)
=&\operatorname{p.v.}
\int_{\mathbb R^D}
\left[
f(\bm{x}_t+\bm{v})-f(\bm{x}_t)
\right]\nonumber\\
&\times\frac{p(\bm{x}_t+\bm{v})}{p(\bm{x}_t)}
\nu_t(d\bm{v}).
\end{align}
where $\operatorname{p.v.}$ denotes the Cauchy principal-value interpretation.}
\end{proposition}
\begin{IEEEproof}
The proof is relegated to Appendix \ref{appendix_1}.
\end{IEEEproof}


\textbf{Remark 1:} The jump integral in Proposition \ref{Proposition_1} is understood as
\begin{align}
&\operatorname{p.v.}\int_{\mathbb R^D}\left[f(\bm{x}_t+\bm{v})-f(\bm{x}_t)\right]\frac{p(\bm{x}_t+\bm{v})}{p(\bm{x}_t)}\nu_t(d\bm{v})\nonumber\\
=&\lim_{\epsilon\rightarrow0}\int_{\|\bm{v}\|>\epsilon}\left[f(\bm{x}_t+\bm{v})-f(\bm{x}_t)\right]\frac{p(\bm{x}_t+\bm{v})}{p(\bm{x}_t)}\nu_t(d\bm{v})\nonumber.
\end{align}
For $\alpha<1$, the integral is locally absolutely integrable under standard smoothness assumptions. For $\alpha\geq 1$, the uncompensated integral is generally not absolutely integrable near the origin and should be interpreted in the Cauchy principal-value sense. To see this, for smooth $f$ and $p(\bm{x}_t)$, we have $f(\bm{x}_t+\bm{v})-f(\bm{x}_t)=\nabla f(\bm{x}_t)^{\top}\bm{v}+O(\Vert\bm{v}\Vert^2)$ and $\frac{p(\bm{x}_t+\bm{v})}{p(\bm{x}_t)}=1+\bm{v}^{\top}\nabla\log p(\bm{x}_t)+O(\Vert\bm{v}\Vert^2)$. The leading first-order term is odd in $v$ and is canceled by the symmetry of $\nu_t$ in the principal-value sense. The remaining second-order terms are locally integrable for $\alpha<2$. This justifies the principal-value form used in Proposition \ref{Proposition_1}.

\textbf{Remark 2:} We specify the explicit expression of the scaled Lévy measure $\nu_t(d\bm{v})$ to facilitate the following analysis. For the $D$-dimensional case, the standard Lévy measure is defined as $\nu(d\bm{v})=C_{D,\alpha}\Vert\bm{v}\Vert^{-D-\alpha}$ \cite{book1}, where
$C_{D,\alpha}=\frac{\alpha2^{\alpha-1}\Gamma(\frac{D+\alpha}{2})}{\pi^{\frac{D}{2}}\Gamma(1-\frac{\alpha}{2})}$. Combined with the effect of $\Phi_S(t)=\sigma_S(t)\bm{I}_D$, we have $\nu_t(d\bm{v})=|\sigma_S(t)|^{\alpha}\nu(d\bm{v})$. In the following, we denote $C_{D,\alpha,\sigma_S}\triangleq|\sigma_S(t)|^{\alpha}C_{D,\alpha}$ to avoid redundancy.

From Proposition \ref{Proposition_1}, we can see that the backward kernel is related to the density ratio $\frac{p(\bm{x}_t+\bm{v})}{p(\bm{x}_t)}$, which contains global information. Therefore, the jump part only belongs to a Markov-type jump process rather than a Lévy process. This process is much more general and is not as mathematically tractable as the forward process. Consequently, the total reversed SDE corresponding to \eqref{SDE} cannot be expressed as the superposition of a modified drift term, a diffusion term, and a stable process. This makes it challenging to design an efficient sampler directly, as in the diffusion case.

Besides guiding the design of the reverse sampler, the backward generator is also useful for understanding discrete diffusion processes. For example, in a standard DDPM driven by a Gaussian transition kernel, $p(\bm{x}_{t-\Delta t}|\bm{x}_t)$ can be modeled by a Gaussian distribution for the following reason. First, with sufficient diffusion, $\bm{x}_T$ can be approximated by a Gaussian distribution. Next, based on the backward Kolmogorov equation, the drift of the reversed SDE is $-(\bm{b}(\bm{x}_t,t)-g(t)^2\nabla\log p(\bm{x}_t,t))dt$, where $g(t)$ is the coefficient of the Wiener process in the forward SDE. At $t=T$, we have $\nabla\log p(\bm{x}_T,T)\propto-\bm{x}_T$. Then, we discretize the continuous time interval with step size $\Delta t$. In this case, $\bm{x}_{T-\Delta t}=\bm{x}_{T}-(\bm{b}(\bm{x}_T,T)+\bm{x}_T)\Delta t+g(t)\sqrt{\Delta t}\bm{w}$, where $\bm{w}\sim\mathcal{N}(\bm{0},\bm{I}_D)$ and $\bm{b}(\bm{x}_t,t)$ is usually set as an affine function of $\bm{x}_t$. This recursive formula is equivalent to the summation of Gaussian-distributed RVs. Therefore, $p(\bm{x}_{t-\Delta t}|\bm{x}_t)$ can be treated as a Gaussian distribution, and the objective loss function can be significantly simplified by regressing the expectation of $\bm{x}_t,t:T\rightarrow0$.

Unfortunately, according to Proposition \ref{Proposition_1}, this simplification does not hold for the process in \eqref{SDE}. In the sequel, we will design a simple loss function that is more suitable for neural network learning.

\subsection{Reverse sampling decomposition}
From the definition of the generator, given the jump kernel $K(d\bm{v})$, the term $\int_{\mathbb{R}^D}(f(\bm{x}+\bm{v})-f(\bm{x}))K(d\bm{v})$ indicates that the jump amplitude follows the distribution $K(d\bm{v})/\int_{\mathbb{R}^D}K(d\bm{v})$. However, the number of small jumps with $\Vert \bm{v}\Vert\leq\epsilon$ is usually infinite due to the singularity of the Lévy measure at the origin. Therefore, we handle small and large jumps separately. Specifically, we decompose the original integral by truncation with threshold $\epsilon$:
\begin{align}\label{jump_decomposition}
&\int_{\mathbb{R}^D}(f(\bm{x}+\bm{v})-f(\bm{x}))K(d\bm{v})\nonumber\\
=&\int_{\Vert \bm{v}\Vert\leq\epsilon}(f(\bm{x}+\bm{v})-f(\bm{x}))K(d\bm{v})\nonumber\\
&+\int_{\Vert \bm{v}\Vert>\epsilon}(f(\bm{x}+\bm{v})-f(\bm{x}))K(d\bm{v})
\end{align}

In other words, jumps with amplitude larger than $\epsilon$ are treated as large jumps. For large jumps, the occurrence frequency is finite. The number of large jumps follows a Poisson distribution with rate $\lambda\triangleq\int_{\Vert \bm{v}\Vert>\epsilon}K(d\bm{v})$. For a sufficiently small time interval $\Delta t$, the occurrence of a large jump can also be approximated by a Bernoulli distribution with probability $\lambda\Delta t$. If a large jump occurs, we need to obtain one jump sample from the normalized distribution $\frac{K(d\bm{v})}{\lambda}$. Otherwise, no large jump occurs within the interval $\Delta t$, and the reversed process reduces to the standard diffusion process.

For the case with $\Vert \bm{v}\Vert\leq\epsilon$, the Lévy measure is not integrable, and there are infinitely many small jumps. Unlike large jump activities, these small jumps cannot be enumerated explicitly. The distribution of small jumps is also complicated, which does not follow a Gaussian or stable distribution due to amplitude truncation. A simple method is to directly ignore small jump activities when the threshold $\epsilon$ is small enough. However, this may introduce considerable approximation error, especially when the small jumps have a nonzero mean. Meanwhile, this effect accumulates when $\Delta t$ is small.

Here, we adopt a more reasonable approach. After truncation, the small jumps have a finite second moment. Thus, we replace the small jump generator by a Gaussian generator matching the first two local moments. This is a weak approximation whose generator error is controlled by the third local moment of the truncated Lévy measure. Namely, the small jump part within $[t,t+\Delta t]$, denoted by $\bm{X}_{t,\Delta t}^{\leq\epsilon}$, can be written as
\begin{equation}\label{Gaussian_representation_small_jump}
\bm{X}_{t,\Delta t}^{\leq\epsilon}\approx \bm{b}_{t,\Delta t}^{\leq\epsilon}\Delta t+\sqrt{\Delta t}(\Sigma_{t,\Delta t}^{\leq\epsilon})^{\frac{1}{2}}\bm{W}
\end{equation}
where $\bm{W}$ is the standard Gaussian RV and
\begin{align}\label{Gaussian_parameter_1}
\bm{b}_{t,\Delta t}^{\leq\epsilon}=&\operatorname{p.v.}\int_{\Vert \bm{v}\Vert\leq\epsilon}\bm{v} K(d\bm{v})
\end{align}

Substituting the state-dependent backward jump kernel into the small jump drift, we have
\begin{align}\label{Gaussian_parameter_1_simplified}
\bm{b}_{t,\Delta t}^{\leq\epsilon}
&=
\operatorname{p.v.}
\int_{\Vert \bm{v}\Vert\leq\epsilon}
\bm{v}
\frac{p(\bm{x}_t+\bm{v})}{p(\bm{x}_t)}
\nu_t(d\bm{v})
\nonumber\\
&=
\int_{\Vert \bm{v}\Vert\leq\epsilon}
\bm{v}
\left(
\frac{p(\bm{x}_t+\bm{v})}{p(\bm{x}_t)}
-1
\right)
\nu_t(d\bm{v})
\nonumber\\
&=
\int_{\Vert \bm{v}\Vert\leq\epsilon}
\bm{v}
\left(
\bm{v}^{\top}\nabla\log p(\bm{x}_t)
+
R_p(\bm{x}_t,\bm{v},t)
\right)
\nu_t(d\bm{v})
\nonumber\\
&=
A_{\nu_t}\nabla\log p(\bm{x}_t)
+
R_{b,\epsilon}(\bm{x}_t,t),
\end{align}
where the truncated second-order moment matrix is defined as $A_{\nu_t}\triangleq\int_{\Vert \bm{v}\Vert\leq\epsilon}\bm{v}\bm{v}^{\top}\nu_t(d\bm{v})$ and the Taylor remainder is $R_{b,\epsilon}(\bm{x}_t,t)\triangleq\int_{\Vert \bm{v}\Vert\leq\epsilon}\bm{v}R_p(\bm{x}_t,\bm{v},t)\nu_t(d\bm{v})$. Here, the second equality follows from the symmetry of $\nu_t$, which gives
\begin{align}
\operatorname{p.v.}
\int_{\Vert \bm{v}\Vert\leq\epsilon}
\bm{v}\nu_t(d\bm{v})
=
\bm{0}.
\end{align}

If $p(\bm{x}_t)$ is locally smooth and strictly positive, then
\begin{align}
\frac{p(\bm{x}_t+\bm{v})}{p(\bm{x}_t)}
=
1
+
\bm{v}^{\top}&\nabla\log p(\bm{x}_t)
+
R_p(\bm{x}_t,\bm{v},t),\nonumber\\
&|R_p(\bm{x}_t,\bm{v},t)|
\leq
C\Vert\bm{v}\Vert^2.
\end{align}

For the isotropic $\alpha$-stable L\'evy measure
$\nu_t(d\bm{v})=C\Vert\bm{v}\Vert^{-D-\alpha}d\bm{v}$, the drift remainder satisfies
\begin{align}
\Vert R_{b,\epsilon}(\bm{x}_t,t)\Vert
&\leq
C
\int_{\Vert \bm{v}\Vert\leq\epsilon}
\Vert\bm{v}\Vert^3
\nu_t(d\bm{v})
\nonumber\\
&=
C\epsilon^{3-\alpha}.
\end{align}

Therefore, by replacing the unknown score $\nabla\log p(\bm{x}_t)$ with the score network $s^{\theta_1}(\bm{x}_t,t)$, the practical approximation becomes
\begin{align}
\bm{b}_{t,\Delta t}^{\leq\epsilon}
\approx
A_{\nu_t}s^{\theta_1}(\bm{x}_t,t).
\end{align}

Similarly, the covariance matrix of the small jump component is given by
\begin{align}\label{Gaussian_parameter_2_simplified}
&\bm{\Sigma}_{t,\Delta t}^{\leq\epsilon}\nonumber\\
&=
\int_{\Vert \bm{v}\Vert\leq\epsilon}
\bm{v}\bm{v}^{\top}
\left(
1
+
\bm{v}^{\top}\nabla\log p(\bm{x}_t)
+
R_p(\bm{x}_t,\bm{v},t)
\right)
\nu_t(d\bm{v})
\nonumber\\
&=A_{\nu_t}+R_{\Sigma,\epsilon}(\bm{x}_t,t).
\end{align}
where the remaining Taylor term satisfies
\begin{align}
\Vert R_{\Sigma,\epsilon}(\bm{x}_t,t)\Vert&\leq C\int_{\Vert \bm{v}\Vert\leq\epsilon}\Vert\bm{v}\Vert^4\nu_t(d\bm{v})=C\epsilon^{4-\alpha}.
\end{align}

Thus, the covariance matrix can be approximated as
\begin{align}\label{Gaussian_parameter_2}
\bm{\Sigma}_{t,\Delta t}^{\leq\epsilon}\approx A_{\nu_t}.
\end{align}

Both \eqref{Gaussian_parameter_1_simplified} and \eqref{Gaussian_parameter_2} admit closed-form expressions based on the polar coordinate transform. For instance,
\begin{align}
A_{\nu_t}=&\int_{\Vert \bm{v}\Vert\leq\epsilon}\bm{v}\bm{v}^{\top}\nu_t(d\bm{v})\nonumber\\
=&\frac{|\mathbb{S}_{D-1}|C_{D,\alpha,\sigma_S}\bm{I}_D}{D}\int_{0}^{\epsilon}r^{1-\alpha}dr\nonumber\\
=&\frac{2\pi^{\frac{D}{2}}C_{D,\alpha,\sigma_S}\epsilon^{2-\alpha}}{D\Gamma(\frac{D}{2})(2-\alpha)}\bm{I}_D,
\end{align}
where $\bm{I}_D$ denotes the $D$-dimentional identity matrix and $\mathbb{S}_{D-1}$ represents the $D-1$-dimensional sphere. Note that this is a generator-level weak approximation rather than an exact Gaussian representation of the small jump sum. Denote the effect of large jumps on $\bm{X}_t$ by $\bm{X}_{t,\Delta t}^{>\epsilon}$. Then, the total jump influence during the reverse sampling process is $\bm{X}_{t,\Delta t}^{\leq\epsilon}+\bm{X}_{t,\Delta t}^{>\epsilon}$. Finally, by combining the externally applied drift term and the reversed diffusion part, the recursive relation between $\bm{X}_{t-\Delta t}$ and $\bm{X}_t$ can be constructed.

\section{Inverse Sampling}\label{section_3}
According to the above analysis, the main challenge of inverse sampling comes from large jump sampling. Therefore, this section mainly focuses on how to generate large jump samples based on the current system state and time instant. First, we explain how to obtain large jump samples. We also show that, from a theoretical perspective, a neural network is not necessary for reverse sampling. These observations help determine the learning target and design the corresponding network structure. Then, two techniques are introduced to effectively reduce the complexity. Finally, since additional conditions are often required to control the generation process in practical applications, we also discuss an approximate observation-guided version for inverse problems.

Before proceeding, we first describe a special family of functions that is important for efficient learning.

In many scenarios, we may want to train a network to regress a target $f(\bm{x})$ that is not directly accessible. However, we can efficiently evaluate only its conditional version $f(\bm{x}|\bm{x}_0)$. Thus, we aim to establish the condition under which optimizing the neural network based on $f(\bm{x}|\bm{x}_0)$ is equivalent to optimizing it based on $f(\bm{x})$. This is the core of the following theorem.

\begin{theorem}[Marginal trainable functions]\label{theorem_1}
\textit{Define the accurate target and conditional target as $l(\bm{x}):\mathbb{R}^D\rightarrow\mathbb{R}^{\tilde{D}}$ and $l(\bm{x}|\bm{x}_0):\mathbb{R}^D\rightarrow\mathbb{R}^{\tilde{D}}$, respectively. Denote the neural network by $f^{\theta}(\bm{x})$, which is used to regress $l(\bm{x})$. Then, the optimization problem $\arg\min_{\theta}\mathbb{E}_{\bm{x}\sim p(\bm{x})}[\Vert f^{\theta}(\bm{x})-l(\bm{x})\Vert^2]$ is equivalent to $\arg\min_{\theta}\mathbb{E}_{\bm{x},\bm{x}_0\sim p(\bm{x},\bm{x}_0)}[\Vert f^{\theta}(\bm{x})-l(\bm{x}|\bm{x}_0)\Vert^2]$ if
\begin{equation}\label{Marginal_trainable_conditions}
l(\bm{x})=\int l(\bm{x}|\bm{x}_0)p(\bm{x}_0|\bm{x})d\bm{x}_0
\end{equation}
and the functions $l(\bm{x})$ satisfying \eqref{Marginal_trainable_conditions} are referred to as marginally trainable functions.}
\end{theorem}
\begin{IEEEproof}
The original optimization problem can be reformulated as follows:
\begin{align}\label{split_original_optimization_problem}
&\mathbb{E}_{\bm{x}\sim p(\bm{x})}[\Vert f^{\theta}(\bm{x})-l(\bm{x})\Vert^2]\nonumber\\
=&\mathbb{E}_{\bm{x}\sim p(\bm{x})}[\Vert f^{\theta}(\bm{x})\Vert^2]-2\mathbb{E}_{\bm{x}\sim p(\bm{x})}[f^{\theta}(\bm{x})l(\bm{x})]\nonumber\\
&+\mathbb{E}_{\bm{x}\sim p(\bm{x})}[\Vert l(\bm{x})\Vert^2]
\end{align}

The third term in \eqref{split_original_optimization_problem} is independent of $\theta$ and can be ignored. For the second term,
\begin{align}\label{transform_original_to_condition}
\mathbb{E}_{\bm{x}\sim p(\bm{x})}[f^{\theta}(\bm{x})l(\bm{x})]=&\int f^{\theta}(\bm{x})\int l(\bm{x}|\bm{x}_0)p(\bm{x}_0|\bm{x})d\bm{x}_0p(\bm{x})d\bm{x}\nonumber\\
=&\int\int f^{\theta}(\bm{x}) l(\bm{x}|\bm{x}_0)p(\bm{x}_0|\bm{x})p(\bm{x})d\bm{x}_0d\bm{x}\nonumber\\
=&\mathbb{E}_{\bm{x},\bm{x}_0\sim p(\bm{x},\bm{x}_0)}[f^{\theta}(\bm{x})l(\bm{x}|\bm{x}_0)]
\end{align}

Plugging \eqref{transform_original_to_condition} into \eqref{split_original_optimization_problem} completes the proof.
\end{IEEEproof}

Theorem \ref{theorem_1} provides sufficient conditions under which a wide range of functions can be used for network training. In fact, Theorem \ref{theorem_1} is general since it includes conventional training frameworks such as score matching and flow matching. For example, the target of score matching is the score function $\nabla\log p(\bm{x}_t)$, which is not accessible in practice. Therefore, the conditional score $\nabla\log p(\bm{x}_t|\bm{x}_0)$ is usually used, together with the simple underlying relationship between $\bm{x}_t$ and $\bm{x}_0$. It can be verified that
\begin{align}\label{score_is_MTF}
\nabla\log p(\bm{x}_t)=&\frac{1}{p(\bm{x}_t)}\nabla\int p(\bm{x}_t|\bm{x}_0)p(\bm{x}_0)d\bm{x}_0\nonumber\\
=&\int \frac{\nabla p(\bm{x}_t|\bm{x}_0)}{p(\bm{x}_t)p(\bm{x}_t|\bm{x}_0)}p(\bm{x}_0)p(\bm{x}_t|\bm{x}_0)d\bm{x}_0\nonumber\\
=&\int \frac{\nabla p(\bm{x}_t|\bm{x}_0)}{p(\bm{x}_t|\bm{x}_0)}p(\bm{x}_0|\bm{x}_t)d\bm{x}_0\nonumber\\
=&\int \nabla\log p(\bm{x}_t|\bm{x}_0)p(\bm{x}_0|\bm{x}_t)d\bm{x}_0
\end{align}

Hence, $\nabla\log p(\bm{x}_t)$ belongs to the class of marginally trainable functions. Similar derivations apply to the training of the conditional vector field in flow matching \cite{paper22}.

\subsection{Sampling of large jumps}
Suppose that the dataset $\bm{x}_{0,j},j=1,\cdots,N$ is available. According to the forward SDE, the expressions of $p(\bm{x}_t|\bm{x}_0)$ and $\frac{p(\bm{x}_t+\bm{v}|\bm{x}_0)}{p(\bm{x}_t|\bm{x}_0)}$ can usually be efficiently obtained. Consequently,
\begin{align}\label{byhand_1}
\frac{p(\bm{x}_t+\bm{v})}{p(\bm{x}_t)}=&\int \frac{p(\bm{x}_t+\bm{v}|\bm{x}_0)}{p(\bm{x}_t|\bm{x}_0)}p(\bm{x}_0|\bm{x}_t)d\bm{x}_0\nonumber\\
=&\int \frac{p(\bm{x}_t+\bm{v}|\bm{x}_0)}{p(\bm{x}_t|\bm{x}_0)}\frac{p(\bm{x}_0)p(\bm{x}_t|\bm{x}_0)}{p(\bm{x}_t)}d\bm{x}_0\nonumber\\
=&\int \frac{p(\bm{x}_t+\bm{v}|\bm{x}_0)}{p(\bm{x}_t|\bm{x}_0)}w(\bm{x}_t,\bm{x}_0)p(\bm{x}_0)d\bm{x}_0\nonumber\\
\approx&\frac{1}{N}\sum_{j=1}^{N}\frac{p(\bm{x}_{t}+\bm{v}|\bm{x}_{0,j})}{p(\bm{x}_{t}|\bm{x}_{0,j})}w(\bm{x}_t,\bm{x}_{0,j})
\end{align}
where we define
\begin{align}\label{byhand_2}
w(\bm{x}_t,\bm{x}_0)\triangleq&\frac{p(\bm{x}_t|\bm{x}_0)}{p(\bm{x}_t)}=\frac{p(\bm{x}_t|\bm{x}_0)}{\int p(\bm{x}_t|\bm{x}_0)p(\bm{x}_0)d\bm{x}_0}\nonumber\\
\approx&\frac{p(\bm{x}_t|\bm{x}_0)}{\frac{1}{N}\sum_{j=1}^{N}p(\bm{x}_t|\bm{x}_{0,j})}
\end{align}

As for the overall jump rate, we have
\begin{align}\label{byhand_3}
\lambda(\bm{x}_t)\triangleq&\int_{\Vert \bm{v}\Vert>\epsilon}\frac{p(\bm{x}_t+\bm{v})}{p(\bm{x}_t)}\nu_t(d\bm{v})\nonumber\\
=&\frac{1}{\sum_{j=1}^{N}p(\bm{x}_t|\bm{x}_{0,j})}\sum_{j=1}^{N}\int_{\Vert \bm{v}\Vert>\epsilon}p(\bm{x}_t+\bm{v}|\bm{x}_{0,j})\nu_t(d\bm{v})
\end{align}

Finally, the unnormalized jump distribution is written as follows,
\begin{align}\label{byhand_4}
q(\bm{x}_t,\bm{v})\triangleq&\frac{p(\bm{x}_t+\bm{v})}{p(\bm{x}_t)}\nu_t(d\bm{v})\nonumber\\
\approx&\frac{\sum_{j=1}^{N}p(\bm{x}_t+\bm{v}|\bm{x}_{0,j})}{\sum_{j=1}^{N}p(\bm{x}_t|\bm{x}_{0,j})}\nu_t(d\bm{v}),\Vert \bm{v}\Vert>\epsilon
\end{align}

The standard score function can also be directly approximated from the dataset based on \eqref{score_is_MTF} as follows:
\begin{align}\label{byhand_5}
\nabla\log p(\bm{x}_t)=&\int \nabla\log p(\bm{x}_t|\bm{x}_0)p(\bm{x}_0|\bm{x}_t)d\bm{x}_0\nonumber\\
\approx&\frac{1}{N}\sum_{j=1}^{N}\frac{w(\bm{x}_t,\bm{x}_{0,j})\nabla p(\bm{x}_t|\bm{x}_{0,j})}{p(\bm{x}_t|\bm{x}_{0,j})}
\end{align}

With the above procedure, reverse sampling can be performed using the distribution $p(\bm{x}_t|\bm{x}_0)$. However, this introduces an unacceptable computational burden during inference. For example, $w(\bm{x}_t,\bm{x}_0)$ and $\lambda(\bm{x}_t)$ need to be calculated at every iteration. This remains time-consuming, even though the explicit expression of $p(\bm{x}_t|\bm{x}_0)$ and the integral in \eqref{byhand_3} are available.

\subsection{Network amortization}
To make large jump sampling more efficient, we use a neural network to amortize part of the computational burden and accelerate inference.

Based on the previous analysis, there are several possible learnable targets, including the density ratio $\frac{p(\bm{x}_t+\bm{v})}{p(\bm{x}_t)}$, the prior distribution $p(\bm{x}_t)$, and the logarithmic density ratio $\log\frac{p(\bm{x}_t+\bm{v})}{p(\bm{x}_t)}$. However, these choices lead to numerical challenges during both learning and sampling. For example, consider learning $\frac{p(\bm{x}_t+\bm{v})}{p(\bm{x}_t)}$. Its advantage is that it is marginally trainable, so the learning target can be easily computed from its conditional version. However, large jump sampling requires numerical integration with respect to $d\bm{v}$ to obtain the overall jump rate. This is infeasible for high-dimensional data, since the integrand value must be obtained from the network output. Even though analytical expressions of $p(\bm{x}_t|\bm{x}_{0,j})$ and $\int_{\Vert \bm{v}\Vert>\epsilon}p(\bm{x}_t+\bm{v}|\bm{x}_{0,j})\nu_t(d\bm{v})$ can be derived, $2N$ additions are still required at each iteration. In addition, numerical stability is not guaranteed when $p(\bm{x}_t)$ is very small but $p(\bm{x}_t+\bm{v})$ is relatively large. In this case, $\frac{p(\bm{x}_t+\bm{v})}{p(\bm{x}_t)}$ behaves like an outlier with a large amplitude.

Although this issue can be alleviated by applying the logarithm, $\log\frac{p(\bm{x}_t+\bm{v})}{p(\bm{x}_t)}$ does not satisfy the conditions in Theorem \ref{theorem_1}. Moreover, the integration difficulty still remains. As for $p(\bm{x}_t)$, it requires the network to learn the absolute values of the global probability distribution, which is difficult, especially when $D$ is large. Note that the method in \cite{paper19} is equivalent to directly learning reverse sampling, and its large jump part is generated by the target distribution. This approach places a heavier burden on the network, making it more like a black box and difficult to control.

Therefore, we choose to learn the large jump rate, i.e., $\lambda(\bm{x}_t)$. This task is similar to parameter estimation for a specialized distribution. From this perspective, it is more learnable than $p(\bm{x}_t)$ and $\frac{p(\bm{x}_t+\bm{v})}{p(\bm{x}_t)}$. Moreover, although $\lambda(\bm{x}_t)$ is the integral of the density ratio with respect to the Lévy measure, it still belongs to the class of marginally trainable functions. Therefore, it can be efficiently trained using $\lambda(\bm{x}_t|\bm{x}_0)\triangleq\int_{\Vert \bm{v}\Vert>\epsilon}\frac{p(\bm{x}_t+\bm{v}|\bm{x}_0)}{p(\bm{x}_t|\bm{x}_0)}\nu_t(d\bm{v})$.

Based on the above discussion, the next proposition provides the loss function for the jump part. For the diffusion part, the standard score matching framework can be directly used.
\begin{proposition}[Loss function for jump activities]\label{Proposition_2}
\textit{The loss function of the neural network for learning the overall jump rate is given as follows:
\begin{align}\label{loss_function_for_jump_rate}
\mathcal{L}(\theta)=&\mathbb{E}_{t\sim\mathcal{U}[0,1],\bm{x}_t\sim p(\bm{x}_t|\bm{x}_0),\bm{x}_0\sim p(\bm{x}_0)}\big[\Vert g^{\theta_2}(\bm{x}_t,t)\nonumber\\
&-\lambda(\bm{x}_t|\bm{x}_0)\Vert^2\big]
\end{align}}
\end{proposition}

\subsection{Efficient calculation of large jump rate}
The remaining two main challenges in bridging the gap between theoretical analysis and practical applications are how to calculate the target value in \eqref{loss_function_for_jump_rate} and how to rapidly sample a large jump from the distribution $q(\bm{x}_t,\bm{v})$. Here, we first focus on the first problem.

So far, we have not imposed specific assumptions on the coefficients in \eqref{SDE}, except that the drift coefficient depends on both the system state and the time index, while the scaling matrices of the Wiener process and the stable process depend only on the time index. For the design of a specific reverse sampling algorithm, its expressions need to be further specified.\textbf{The principle is to make the learning process as simple as possible without sacrificing much generality.} Similar to conventional diffusion models and SDEs driven only by the Wiener process \cite{paper1,paper3}, we make the following assumptions.
\begin{assumption}\label{assumption_5}
$\bm{b}(\bm{X}_t,t)$ is an affine function of $\bm{X}_t$, i.e., $\bm{b}(\bm{X}_t,t)=R\bm{X}_t+\bm{s}$, which is similar to the Ornstein-Uhlenbeck (OU) process. Since $\bm{b}(\bm{X}_t,t)$ is a time-homogeneous drift, we rewrite it as $\bm{b}(\bm{X}_t)$ in the sequel.
\end{assumption}
\begin{assumption}\label{assumption_6}
For $\Phi_G(t)$ and $\Phi_S(t)$, we restrict them to diagonal matrices, i.e., $\Phi_G(t)=\text{diag}(\sigma_{G,1}(t),\cdots,\sigma_{G,D}(t))$ and $\Phi_S(t)=\text{diag}(\sigma_{S,1}(t),\cdots,\sigma_{S,D}(t))$.
\end{assumption}

In the following, we rewrite $\bm{b}(\bm{X}_t,t)$ as $\bm{b}(\bm{X}_t)$ because it is directly related only to $\bm{X}_t$. Under these configurations, the distribution of $\bm{X}_t$ generated from a given $\bm{X}_0$ can be explicitly determined.

\begin{proposition}[Distribution of $\bm{X}_t$]\label{Proposition_3}
\textit{Consider the forward SDE in \eqref{SDE}. Let $\bm{b}(\bm{X}_t)$, $\Phi_G(t)$, and $\Phi_S(t)$ satisfy Assumptions \ref{assumption_5} and \ref{assumption_6}. Then, $\bm{X}_t=\bm{\mu}(t,\bm{x}_0)+\bm{G}(t)+\bm{S}(t)$, where $\bm{G}(t)\sim\mathcal{N}(\bm{0},\Sigma_G)$ and $\bm{S}(t)$ follows an $\alpha$-stable distribution with characteristic function $\phi_S(t)$. Moreover,
\begin{equation}\label{distribution_parameter_1}
\bm{\mu}(t,\bm{x}_0)=\exp(Rt)\bm{x}_0+\exp(Rt)\int_{0}^{t}\exp(-Rl)\bm{s}dl
\end{equation}
\begin{equation}\label{distribution_parameter_2}
\Sigma_G(t)=\int_{0}^{t}\exp((t-l)R)\Phi_G(l)\Phi_G(l)^\top\exp((t-l)R^\top)dl
\end{equation}
\begin{equation}\label{distribution_parameter_3}
\phi_S(\bm{u},t)=\exp\bigg(-\int_{0}^{t}\big\Vert\Phi_S(l)^\top\exp((t-l)R^\top)\bm{u}\big\Vert_{2}^{\alpha}dl\bigg)
\end{equation}}
\end{proposition}
\begin{IEEEproof}
The proof is relegated to Appendix \ref{appendix_2}.
\end{IEEEproof}

\begin{corollary}[Distribution of $\bm{X}_t$ for special cases]\label{Corollary_1}
\textit{Assume that $R$, $\Phi_G(t)$, and $\Phi_S(t)$ are all scaled identity matrices, i.e., $R=R_0\bm{I}_D$, $\Phi_G(t)=\sigma_G(t)\bm{I}_D$, and $\Phi_S(t)=\sigma_S(t)\bm{I}_D$. The other settings are the same as those in Proposition \ref{Proposition_3}. Then, $\Sigma_G(t)=\gamma_G(t)^2\bm{I}_D$, where $\gamma_G(t)=\sqrt{\int_{0}^{t}\exp(2R_0(t-l))\sigma_{G}(l)^2dl}$. Moreover, $\phi_S(\bm{u},t)=\exp(-\gamma_{\alpha}(t)^{\alpha}\Vert\bm{u}\Vert_{2}^{\alpha})$, where $\gamma_{\alpha}(t)=(\int_{0}^{t}|\sigma_S(l)|^{\alpha}\exp(\alpha R_0(t-l))dl)^{\frac{1}{\alpha}}$.}
\end{corollary}

\textbf{Remark 3:} From Proposition \ref{Proposition_3}, it can be observed that the final distribution of $\bm{X}_T$ is still related to $\bm{X}_0$. In practice, we can choose $T$ and $R<0$ such that $\exp(RT)\bm{X}_0\approx0$. In this case, the original distribution of the inverse sampling can be set to the mixed distribution composed of Gaussian distribution and $\alpha$-stable distribution.

Proposition \ref{Proposition_3} and Corollary \ref{Corollary_1} imply that if $R$, $\Phi_G(t)$, and $\Phi_S(t)$ are diagonal matrices with different diagonal elements, the isotropic property of the distribution of $\bm{X}_t$ will be destroyed. This makes the following derivations much more cumbersome. In the remainder of this paper, we use the assumptions in Corollary \ref{Corollary_1} to facilitate the design of the reverse sampler.

Recall that the first challenge comes from the integral operation. During training, performing high-dimensional numerical integration for every data sample is time-consuming. Our goal is to derive an analytical expression for $\int_{\Vert \bm{v}\Vert>\epsilon}p(\bm{x}_t+\bm{v}|\bm{x}_0)\nu_t(d\bm{v})$. According to Proposition \ref{Proposition_3}, $p(\bm{x}_t|\bm{x}_0)$ is the distribution of an RV composed of a bias term, a Gaussian RV, and a stable RV, which are mutually independent. Due to the general lack of an analytical PDF for the $\alpha$-stable distribution, it is challenging to obtain an explicit expression for $p(\bm{x}_t|\bm{x}_0)$. Therefore, PDF approximation is needed to facilitate the following analysis.

Let $X_B\sim\mathcal{N}(0,2\gamma_g^2)$ and $X_S\sim\mathcal{S}(\alpha,0,\gamma_s,0)$ be mutually independent. Denote the PDFs of $X_B$ and $X_S$ by $f_{X_B}(x)$ and $f_{X_S}(x)$, respectively. Then, the exact PDF of $X_M\triangleq X_B+X_S$ equals the convolution of $f_{X_B}(x)$ and $f_{X_S}(x)$. Unfortunately, there is no general explicit expression for $f_{X_M}(x)$. In \cite{paper23}, an approximate PDF is proposed to describe the statistical model of $X_M$. It can be regarded as a weighted sum of a Gaussian kernel and a tail component. Specifically,
\begin{equation}\label{approximation_of_X_M_1_dim}
f_{X_M}(x)\approx\tilde{f}_{X_M}(x)=\frac{1}{k_1}\bigg[c_1g_0\exp\bigg(-\frac{x^2}{4\gamma_g^2}\bigg)+\frac{\alpha\gamma_sC_{\alpha}}{c_2+|x|^{\alpha+1}}\bigg],
\end{equation}
where $k_1$ is the normalization factor. The parameter $c_2\geq0$ is used to avoid singularity at the origin and control the shape of the tail component. The parameters $\gamma_g$ and $\gamma_s$ represent the scaling parameters of the Gaussian kernel and the tail component, respectively. The quantities $g_0$ and $C_{\alpha}$ are expressed by the above parameters in closed form. It has been shown in \cite{paper23} that $\tilde{f}_{X_M}(\bm{x})$ can accurately model $X_M$ under various scenarios.

The tail component in \eqref{approximation_of_X_M_1_dim} is derived from the asymptotic behavior of the $\alpha$-stable distribution \cite{paper23,book3}. Following a similar procedure, \eqref{approximation_of_X_M_1_dim} can also be extended to high-dimensional cases \cite{paper24}. For example, let $\bm{X}_B\sim\mathcal{N}(\bm{0},\sigma^2\bm{I})$, $\bm{X}_S=\sqrt{A}\bm{X}_G$, and $\bm{X}_M=\bm{X}_B+\bm{X}_S$. Here, we consider the sub-Gaussian representation of the multivariate isotropic stable distribution, where $A$ follows a right-skewed $\alpha$-stable distribution \cite{book3}. Note that the characteristic function of $\bm{X}_M$ is $\phi_{\bm{X}_M}(\bm{l})=\exp(-\gamma_g^2\Vert\bm{l}\Vert^2-\gamma_s^{\alpha}\Vert\bm{l}\Vert^{\alpha})$. Then, the PDF of $\bm{X}_M$ can be similarly approximated as \cite{paper24}
\begin{align}\label{approxiamted_p_xt}
\tilde{f}_{\bm{X}_M}(\bm{x})=&\frac{1}{k_2}\bigg[c_1g_0\exp\bigg(-\frac{\Vert \bm{x}\Vert^2}{4\gamma_g^2}\bigg)\nonumber\\
&+\frac{\alpha2^{\alpha-1}\Gamma(\frac{\alpha+D}{2})\gamma_s^{\alpha}}{\pi^{D/2}\Gamma(1-\frac{\alpha}{2})}(c_2+\Vert \bm{x}\Vert^{2})^{-\frac{\alpha+D}{2}}\bigg]
\end{align}
where $k_2$ denotes the normalization factor.

With \eqref{approxiamted_p_xt}, the target in \eqref{loss_function_for_jump_rate} can be further derived. Since $p(\bm{x}_t|\bm{x}_{0,j})$ is independent of $d\bm{v}$, we focus only on $\int_{\Vert\bm{v}\Vert>\epsilon}p(\bm{x}_t+\bm{v}|\bm{x}_0)\nu_t(d\bm{v})$ in the sequel. For notational simplicity, we derive explicit expressions for the following integrals:
\begin{equation}\label{objective_integral_1}
P_1(\epsilon)\triangleq\int_{\Vert\bm{v}\Vert>\epsilon}\exp(-\Vert\bm{v}+\bm{\mu}\Vert^2)\nu_t(d\bm{v})
\end{equation}
\begin{equation}\label{objective_integral_2}
P_2(\epsilon)\triangleq\int_{\Vert\bm{v}\Vert>\epsilon}(c_2+\Vert\bm{v}+\bm{\mu}\Vert^2)^{-\frac{\alpha+D}{2}}\nu_t(d\bm{v})
\end{equation}

Note that $\bm{\mu}$ is introduced only to characterize the mismatch between $p(\bm{x}_t+\bm{v}|\bm{x}_0)$ and the Lévy measure. Moreover, the scaling parameter $\gamma_g$ is omitted in \eqref{objective_integral_1} for brevity. It does not denote the mean of $\bm{X}_t$. Direct numerical integration is infeasible in high-dimensional settings. To address this issue, we transform $P_1(\epsilon)$ and $P_2(\epsilon)$ into polar coordinates as follows:
\begin{align}\label{derive_I_1_1}
P_1(\epsilon)=&\frac{2\pi^{\frac{D-1}{2}}C_{D,\alpha,\sigma_S}}{\Gamma(\frac{D-1}{2})}\int_{\epsilon}^{+\infty}\int_{-1}^{1}\exp(-(r^2+\Vert\bm{\mu}\Vert^2\nonumber\\
&+2rz\Vert\bm{\mu}\Vert))r^{-\alpha-1}(1-z^2)^{\frac{D-3}{2}}dzdr
\end{align}
\begin{align}\label{derive_I_2_1}
P_2(\epsilon)=&\frac{2\pi^{\frac{D-1}{2}}C_{D,\alpha,\sigma_S}}{\Gamma(\frac{D-1}{2})}\int_{\epsilon}^{+\infty}\int_{-1}^{1}(c_2+r^2+\Vert\bm{\mu}\Vert^2\nonumber\\
&+2rz\Vert\bm{\mu}\Vert)^{-\frac{\alpha+D}{2}}r^{-\alpha-1}(1-z^2)^{\frac{D-3}{2}}dzdr
\end{align}
which does not admit a closed-form expression in general. However, the original integral has been significantly simplified into a two-dimensional integral. In fact, the inner integral with respect to $z$ in \eqref{objective_integral_1} and \eqref{objective_integral_2} can be further simplified. First, we need the following two integral identities \cite{book2}:
\begin{align}\label{integral_identity_1}
\int_{-1}^{1}\exp(ax)&(1-x^2)^pdx\nonumber\\
=&\sqrt{\pi}\Gamma(p+1)(2a^{-1})^{\frac{p+1}{2}}I_{\frac{p+1}{2}}(a),
\end{align}
\begin{align}\label{integral_identity_2}
\int_{0}^{1}x^{a-1}&(1-x)^{b-a-1}(1-zx)^{-c}dx\nonumber\\
=&B(a,b-a)\prescript{}{2}F_{1}(c,a;b;z),b>a>0
\end{align}
where $I_{a}(x)$ denotes the modified Bessel function of the first kind and $B(\cdot)$ denotes the Beta function. For $P_1(\epsilon)$, applying \eqref{integral_identity_1} yields
\begin{align}\label{derive_I_1_2}
P_1(\epsilon)=&2\pi^{\frac{D}{2}}C_{D,\alpha,\sigma_S}\Vert\bm{\mu}\Vert^{\frac{2-D}{2}}\exp(-\Vert\bm{\mu}\Vert^2)\nonumber\\
&\times\int_{\epsilon}^{+\infty}\exp(-r^2)r^{-\alpha-\frac{D}{2}}I_{\frac{D-2}{2}}(2r\Vert\bm{\mu}\Vert)dr
\end{align}

\begin{figure*}[tb]
\begin{align}\label{derive_I_2_2}
P_2(\epsilon)=&\frac{C_{D,\alpha,\sigma_S}\pi^{\frac{D-1}{2}}2^{D-1}B(\frac{D-1}{2},\frac{D-1}{2})}{\Gamma(\frac{D-1}{2})}\int_{\epsilon}^{+\infty}\frac{r^{-\alpha-1}\prescript{}{2}F_{1}(\frac{\alpha+D}{2},\frac{D-1}{2};D-1;-\frac{4r\Vert\bm{\mu}\Vert}{c_2+r^2+\Vert\bm{\mu}\Vert^2-2r\Vert\bm{\mu}\Vert})}{(c_2+r^2+\Vert\bm{\mu}\Vert^2-2r\Vert\bm{\mu}\Vert)^{\frac{\alpha+D}{2}}}dr
\end{align}
\hrulefill
\end{figure*}

For $P_2(\epsilon)$, we use the variable transformation $x=2t-1$. After some manipulations, \eqref{derive_I_2_2} can be obtained. In this way, calculating the total jump rate $\lambda(\bm{x}_t|\bm{x}_{0,j})$ reduces to evaluating a one-dimensional integral, which significantly decreases the computational complexity compared to $D$-dimensional integrals.


\subsection{High-dimensional sampling}
As for the second problem, numerical sampling algorithms such as Hamiltonian Monte Carlo (HMC) can be applied, as the PDF and its gradient information are available. However, reverse sampling of the SDE is still slow due to the cold-start problem. Specifically, for each reverse sampling iteration with sufficiently small $\Delta t$, only one sample is needed for a large jump if it occurs. However, the target sampling distribution varies with $t$. In this case, the HMC sampler should be ``retrained'' repeatedly to reach the steady state of the dynamic system, which is quite inefficient.

Then, we design the fast sampling method for the large jump. Recall that the normalized target distribution is
\begin{align}\label{normalizing_objective_distribution}
\tilde{q}(\bm{x}_t,\bm{v})\triangleq&\frac{p(\bm{x}_t+\bm{v})\nu_t(d\bm{v})}{Q_{t,\epsilon}}\nonumber\\
\approx&\frac{\sum_{j=1}^{N}p(\bm{x}_t+\bm{v}|\bm{x}_{0,j})\nu_t(d\bm{v})}{NQ_{t,\epsilon}}
\end{align}
where we define $Q_{t,\epsilon}=\int_{\Vert \bm{v}\Vert>\epsilon}p(\bm{x}_t+\bm{v})\nu_t(d\bm{v})$. Note that $p(\bm{x}_t)$ is omitted because it is not related to $\bm{v}$, and our objective is to generate a sample from the distribution with respect to $\bm{v}$.

Different from the one-dimensional case, we sample from $\tilde{q}(\bm{x}_t,\bm{v})$ under polar coordinates. First, we rewrite $\tilde{q}(\bm{x}_t,\bm{v})$ as $\tilde{q}(\bm{x}_t,r,z,\theta)$, where $r\triangleq\Vert\bm{v}\Vert$, $z\triangleq\frac{\langle\bm{\mu},\bm{v}\rangle}{\Vert\bm{\mu}\Vert\Vert\bm{v}\Vert}$, and $\theta\in\mathcal{R}^{D-2}$ represents the remaining angular variables. Then, $\tilde{q}(\bm{x}_t,\bm{v})$ can be rewritten as in \eqref{equivalent_q}, where $\bm{\mu}$ is defined as before and incorporates the information of $\bm{x}_t$. We then need to sample $r$, $z$, and $\theta$ from the joint distribution $\tilde{q}(\bm{x}_t,r,z,\theta)$. Given $r$ and $z$, $\theta$ can be generated uniformly on the sphere orthogonal to the direction represented by $z$. Therefore, $\theta$ does not explicitly appear in the PDF. Based on this observation, we omit $\theta$ in $\tilde{q}(\bm{x}_t,r,z,\theta)$ for simplicity. 

Note that the derivations of the marginal distributions are the same as those of $P_1(\epsilon)$ and $P_2(\epsilon)$. Hence, the previous lookup tables can be directly used for fast sampling based on inverse CDF methods. In other words, we can first sample $r$ from the marginal distribution $\tilde{q}(\bm{x}_t,r)$ and then sample $z$ from $\tilde{q}(z|\bm{x}_t,r)$. To distinguish the distribution under polar coordinates from $\tilde{q}(\bm{x}_t,\bm{v})$, we denote the CDFs of $\tilde{q}(\bm{x}_t,r)$ and $\tilde{q}(z|\bm{x}_t,r)$ by $F_{\tilde{q},(r,t)}(\cdot)$ and $F_{\tilde{q},z|(r,t)}(\cdot)$, respectively, in the sequel.

However, even if $P_1(x)$ and $P_2(x)$ with $x\in[\epsilon,M]$ and a large constant $M$ can be precomputed and stored in the lookup table, the inference stage is still time-consuming. This is because, for each pair $(r,F_{\tilde{q},(r,t)}(\cdot))$, the corresponding conditional pair $(r,F_{\tilde{q}(|),(r,t)}(\cdot))$ needs to be obtained $N$ times. This complexity can be considerably reduced by another simple sampling step. Indeed, $Q_{t,\epsilon}$ can be decomposed as
\begin{align}\label{decompose_Q}
Q_{t,\epsilon}=\frac{1}{N}\sum_{j=1}^{N}\int_{\Vert \bm{v}\Vert>\epsilon} p(\bm{x}_t+\bm{v}|\bm{x}_{0,j})\nu_t(d\bm{v})\triangleq \frac{1}{N}\sum_{j=1}^{N}Q_{t,\epsilon,j}
\end{align}

Based on \eqref{normalizing_objective_distribution}, $\tilde{q}(\bm{x}_t,\bm{v})$ can be rewritten as follows:
\begin{align}\label{rewritte_objective_PDF}
\tilde{q}(\bm{x}_t,\bm{v})=&\frac{\sum_{j=1}^{N}p(\bm{x}_t+\bm{v}|\bm{x}_{0,j})\nu_t(d\bm{v})}{\sum_{j=1}^{N}Q_{t,\epsilon,j}}\nonumber\\
=&\frac{\sum_{j=1}^{N}Q_{t,\epsilon,j}\tilde{q}(\bm{x}_t,\bm{v}|\bm{x}_{0,j})}{\sum_{j=1}^{N}Q_{t,\epsilon,j}}\nonumber\\
\triangleq&\sum_{j=1}^{N}g_Q(\bm{x}_{0,j}|\bm{x}_t,t)\tilde{q}(\bm{x}_t,\bm{v}|\bm{x}_{0,j})
\end{align}
where we define the weighted distribution $g_Q(\bm{x}|\bm{x}_t,t)$ as
\begin{equation}
g_Q(\bm{x}|\bm{x}_t,t)=\sum_{j=1}^{N}\frac{1_{\{\bm{x}_{0,j}\}}(\bm{x})Q_{t,\epsilon,j}}{\sum_{j=1}^{N}Q_{t,\epsilon,j}}
\end{equation}

In this case, $\tilde{q}(\bm{x}_t,\bm{v})$ can be regarded as a mixture distribution of the conditional distributions $\tilde{q}(\bm{x}_t,\bm{v}|\bm{x}_{0,j})$. Then, $\tilde{q}(\bm{x}_t,\bm{v})$ is equivalent to a distribution controlled by an underlying Markov process. Based on the above observation, we first randomly choose a data sample $\bm{x}^*_{0}\in \{\bm{x}_{0,j},j=1,\cdots,N\}$ according to the distribution $g_Q(\bm{x}|\bm{x}_t,t)$. Then, we generate a sample from the distribution conditioned on $\bm{x}^*_{0}$. In this way, we only need to focus on the conditional distribution $\tilde{q}(\bm{x}_t|\bm{x}^*_{0})$ and obtain the pair $(r,F_{\tilde{q}(|),(r,t)}(\cdot))$ only once. This is much faster than directly sampling from $\tilde{q}(\bm{x}_t,\bm{v})$. 

To obtain the complete $g_Q(\bm{x}|\bm{x}_t,t)$, all $Q_{t,\epsilon,j}$ with $j=1,\cdots,N$ should be computed. This is equivalent to searching the lookup table $N$ times. In fact, many $Q_{t,\epsilon,j}$ may be rather small, and their influence can be neglected. Therefore, instead of constructing the complete $g_Q(\bm{x}|\bm{x}_t,t)$, we truncate $g_Q(\bm{x}|\bm{x}_t,t)$ by choosing the $K$ samples with the largest approximate mixture weights, or equivalently the largest $Q_{t,\epsilon,j}$ in the prior sampler. This operation is reasonable because the contribution of $\bm{x}_{0,j}$ is determined by the overlap between the conditional density $p(\bm{x}_t+\bm{v}|\bm{x}_{0,j})$ and the truncated Lévy measure. In the isotropic case, $\nu_t(d\bm{v})\propto \Vert\bm{v}\Vert^{-D-\alpha}d\bm{v}$. After truncation, it favors jumps with $\Vert\bm{v}\Vert$ close to $\epsilon$. Therefore, $Q_{t,\epsilon,j}$ is large only when the high-density region of $p(\bm{x}_t+\bm{v}|\bm{x}_{0,j})$ intersects the Lévy-favored shell around $\bm{x}_t$. This motivates a top-$K$ approximation that keeps only the candidates with the largest approximate $Q_{t,\epsilon,j}$ and renormalizes their weights. For convenience, we denote the index set by $\mathcal{K}$ and the truncated version of $g_Q(\bm{x}|\bm{x}_t,t)$ by $\tilde{g}_Q(\bm{x})$.

\textbf{Remark 4:} If we choose the largest $K$ elements from all $Q_{t,\epsilon,j}$ via the trivial traversal, the complexity still remains $O(N)$. To solve this problem, more advanced and efficient ordering algorithms can be utilized, such as approximate nearest-neighbor (ANN) methods.

\begin{figure*}[tb]
\begin{align}\label{equivalent_q}
\tilde{q}(\bm{x}_t,r,z,\theta)\propto&\bigg(\frac{c_1g_0}{k_2}\exp\bigg(-\frac{1}{4\gamma_g^2}(r^2+\Vert\bm{\mu}\Vert^2+2r z\Vert\bm{\mu}\Vert)\bigg)+\frac{\alpha2^{\alpha-1}\Gamma(\frac{\alpha+D}{2})\gamma_s^{\alpha}}{k_2\pi^{D/2}\Gamma(1-\frac{\alpha}{2})}(c_2+r^2+\Vert\bm{\mu}\Vert^2+2r z\Vert\bm{\mu}\Vert)^{-\frac{\alpha+D}{2}}\bigg)\nonumber\\
&\times r^{-\alpha-1}(1-z^2)^{\frac{D-3}{2}}, r>\epsilon,z\in[-1,1],
\end{align}
\hrulefill
\end{figure*}

In addition to the inverse CDF method, rejection sampling can also be efficient for sampling the radial variable $r$ based on \eqref{equivalent_q}. This converts the original $D$-dimensional sampling problem into two one-dimensional sampling steps and one remaining direction sampling step. Since \eqref{equivalent_q} has an analytical PDF expression with explicit asymptotic behavior, the proposal distribution for sampling $r$ can be chosen as a Student's t-distribution with $2\alpha+D$ degrees of freedom. This choice matches the tail behavior of the marginal distribution induced by \eqref{equivalent_q}. Note that marginalization with respect to $z$ does not change the tail order of the distribution in terms of $r$. Its closed-form expression has been analyzed in \eqref{derive_I_1_1}$\sim$\eqref{derive_I_2_2}. For $z$, rejection sampling is not necessary, since the proposal distribution is not straightforward to design due to the bounded domain, and the corresponding lookup table is quite small.

Finally, the overall training procedure, the large jump sampling method, and the complete reverse sampling algorithm are summarized in Algorithm \ref{algorithm_1}, Algorithm \ref{algorithm_2}, and Algorithm \ref{algorithm_3}, respectively. Note that complete reverse sampling starts from samples generated from the terminal prior distribution $p_{tp}(\bm{x})$. In our setting, $p_{tp}(\bm{x})$ is the convolution of a multivariate Gaussian distribution and a stable distribution. The corresponding samples can be generated separately from these two distributions. Here, if we set $\bm{s}\neq\bm{0}$, the mean of $p_{tp}(\bm{x})$ is $-\frac{\bm{s}}{R_0}$. Without loss of generality, we can also set $\bm{s}=\bm{0}$ so that $p_{tp}(\bm{x})$ is symmetric about the origin.

\textbf{Remark 5:} In Algorithm \ref{algorithm_2}, we only consider the case where $\bm{\mu}_t\neq\bm{0}$ and $D\geq3$. If $\bm{\mu}_t=\bm{0}$, the conditional density is radially symmetric. Therefore, we only need to sample $r$. When $D=1,2$, the sampling procedure can be significantly simplified. For example, when $D=1$, the orthogonal direction does not exist, and only $r$ needs to be sampled.

\textbf{Remark 6:} We emphasize that the generator-based reverse characterization is stated at the level of Lévy-driven Markov processes, whereas the computational sampler developed in this paper is derived under additional structural assumptions. In particular, the practical algorithm focuses on isotropic linear Lévy SDEs with symmetric $\alpha$-stable jump measures, for which the conditional transition density, jump rate approximation, and polar coordinate sampling procedure can be implemented efficiently. This distinction allows us to retain the general insight provided by the reversed generator while avoiding an overstatement of the scope of the proposed numerical sampler.

\begin{algorithm}[!ht]
\renewcommand{\algorithmicrequire}{\textbf{Input:}}
\renewcommand{\algorithmicensure}{\textbf{Output:}}
\caption{Training algorithm}\label{algorithm_1}
\begin{algorithmic}[1]
\Require Dataset $\bm{x}_{0,j},j=1,\cdots,N$ and the training epoch $N_{\text{epoch}}$.  
\Ensure Well-trained networks $s^{\theta_1}(\cdot)$ and $g^{\theta_2}(\cdot)$. 
\State \textbf{\# Training of $s^{\theta_1}(\cdot)$:}
\State \hspace{1em}Use the dataset to optimize $\theta_1$ via standard denoising score matching.
\State \textbf{\# Training of $g^{\theta_2}(\cdot)$:}
\State \hspace{1em}$k\leftarrow1$.
\State \hspace{1em}\textbf{while} {$k\leq N_{\text{epoch}}$} \textbf{do}
\State \hspace{2em}Sample $t\sim\mathcal{U}(0,1)$.
\State \hspace{2em}Calculate the bias and scaling parameters of the Gaussian and stable RVs used to describe $\bm{x}_t$ based on \eqref{distribution_parameter_1}$\sim$\eqref{distribution_parameter_3}.
\State \hspace{2em}Calculate the target based on \eqref{loss_function_for_jump_rate}, \eqref{derive_I_1_2}, and \eqref{derive_I_2_2}.
\State \hspace{2em}Optimize $\theta_2$ by stochastic gradient descent based on \eqref{loss_function_for_jump_rate}.
\State \hspace{2em}$k\leftarrow k+1$.
\State \hspace{1em}\textbf{end while}
\end{algorithmic}
\end{algorithm}

\begin{algorithm}[!ht]
\renewcommand{\algorithmicrequire}{\textbf{Input:}}
\renewcommand{\algorithmicensure}{\textbf{Output:}}
\caption{Sampling algorithm for large jumps}\label{algorithm_2}
\begin{algorithmic}[1]
\Require $\bm{x}_t$, $\bm{x}^*_{0}$, lookup tables containing $(r,F_{\tilde{q},(r,t)}(\cdot))$ and $(z,F_{\tilde{q},z|(r,t)}(\cdot))$. 
\Ensure Large jump sample $\bm{\omega}$. 
\State \hspace{0em}$\bm{\mu}_t\leftarrow\bm{x}_t-\bm{\mu}(t,\bm{x}_0^*)$, where $\bm{\mu}(t,\bm{x}_0^*)$ is given in \eqref{distribution_parameter_1}.
\State \hspace{0em}$\breve{\bm{\mu}}_t\leftarrow\frac{\bm{\mu}_t}{\Vert\bm{\mu}_t\Vert}$.
\State \hspace{0em}Sample $u_1,u_2\sim\mathcal{U}(0,1)$ and $\bm{u}_3\sim\mathcal{N}(\bm{0},\bm{I}_D)$.
\State \hspace{0em}$r\leftarrow F^{-1}_{\tilde{q},(r,t)}(u_1)$, or equivalently, apply rejection sampling based on $\tilde{q}(\bm{x}_t,r)$ with Student's t-distribution as the proposal distribution.
\State \hspace{0em}$z\leftarrow F^{-1}_{\tilde{q},z|(r,t)}(u_2)$.
\State \hspace{0em}$\bm{u}_3\leftarrow \bm{u}_3-\langle\breve{\bm{\mu}}_t,\bm{u}_3\rangle\breve{\bm{\mu}}_t$.
\State \hspace{0em}$\breve{\bm{u}}_3\leftarrow \frac{\bm{u}_3}{\Vert\bm{u}_3\Vert}$.
\State \hspace{0em}$\bm{\omega}\leftarrow rz\breve{\bm{\mu}}_t+r\sqrt{1-z^2}\breve{\bm{u}}_3$.
\end{algorithmic}
\end{algorithm}

\begin{algorithm}[!ht]
\renewcommand{\algorithmicrequire}{\textbf{Input:}}
\renewcommand{\algorithmicensure}{\textbf{Output:}}
\caption{Total reverse sampling algorithm}\label{algorithm_3}
\begin{algorithmic}[1]
\Require $\bm{x}_{T,j}\sim p_{tp}(\bm{x}),j=1,\cdots,N$, well-trained networks $s^{\theta_1}(\cdot)$ and $g^{\theta_2}(\cdot)$, weighted distribution $g_Q(\bm{x}|\bm{x}_t,t)$, lookup tables containing $(r,F_{\tilde{q},(r,t)}(\cdot))$ and $(z,F_{\tilde{q},z|(r,t)}(\cdot))$, and time step size $\Delta t$. 
\Ensure Reversed samples $\bm{x}_{0,j},j=1,\cdots,N$. 
\State $j\leftarrow1$.
\State \textbf{while} {$j\leq N$} \textbf{do}
\State \hspace{1em}$t\leftarrow T$.
\State \hspace{1em}\textbf{while} {$t>0$} \textbf{do}
\State \hspace{2em}\textbf{\# Score-based update:}
\State \hspace{3em}Sample $\bm{\xi}_1,\bm{\xi}_2\sim\mathcal{N}(\bm{0},\bm{I}_D)$.
\State \hspace{3em}$\tilde{\bm{x}}_{t,j}\leftarrow \bm{x}_{t,j}+(-\bm{b}(\bm{x}_{t,j})+\Sigma(t)s^{\theta_1}(\bm{x}_{t,j},t))\Delta t+\Phi_G(t)\sqrt{\Delta t}\bm{\xi}_1$.
\State \hspace{2em}\textbf{\# Large jump generation:}
\State \hspace{3em}$\lambda(\bm{x}_{t,j})\leftarrow g^{\theta_2}(\bm{x}_{t,j},t)$.
\State \hspace{3em}Sample $\beta\sim\mathcal{U}(0,1)$.
\State \hspace{3em}\textbf{if} $\beta<1-\exp(-\lambda(\bm{x}_{t,j})\Delta t)$ \textbf{then}
\State \hspace{4em}Randomly choose a sample $\bm{x}^*_{0}$ from the dataset according to the truncated weighted distribution $\tilde{g}_Q(\bm{x})$.
\State \hspace{4em}Obtain a jump sample $\bm{\omega}$ from Algorithm \ref{algorithm_2} using $\bm{x}_{t,j}$, $\bm{x}^*_{0}$, and the lookup tables.
\State \hspace{3em}\textbf{else}
\State \hspace{4em}$\bm{\omega}\leftarrow\bm{0}$.
\State \hspace{3em}\textbf{end if}
\State \hspace{3em}$\tilde{\bm{x}}_{t,j}\leftarrow\tilde{\bm{x}}_{t,j}+\bm{\omega}$.
\State \hspace{2em}\textbf{\# Small jump generation:}
\State \hspace{3em}Calculate the drift $\bm{b}_{t,\Delta t}^{\leq\epsilon}$ in \eqref{Gaussian_parameter_1_simplified} and the scaling parameter $\Sigma_{t,\Delta t}^{\leq\epsilon}$ in \eqref{Gaussian_parameter_2}.
\State \hspace{3em}$\bm{x}_{t-\Delta t,j}\leftarrow \tilde{\bm{x}}_{t,j}+\bm{b}_{t,\Delta t}^{\leq\epsilon}\Delta t+\sqrt{\Delta t}(\Sigma_{t,\Delta t}^{\leq\epsilon})^{\frac{1}{2}}\bm{\xi}_2$.
\State \hspace{3em}$t\leftarrow t-\Delta t$.
\State \hspace{1em}\textbf{end while}
\State \hspace{1em}$j\leftarrow j+1$.
\State \hspace{0em}\textbf{end while}
\end{algorithmic}
\end{algorithm}

It should be remarked that for the large jump component, we should theoretically sample the number of jumps from $N_{t}^{>\epsilon}\sim\operatorname{Poisson}\left(g^{\theta_2}(\bm{x}_t,t)\Delta t\right)$. If $N_{t}^{>\epsilon}>0$, independently generate jump amplitudes $\bm{V}_{t,\ell}^{>\epsilon},\ell=1,\ldots,N_t^{>\epsilon}$. Then the large jump contribution is $\bm{X}_{t,\Delta t}^{>\epsilon}=\sum_{\ell=1}^{N_t^{>\epsilon}}\bm{V}_{t,\ell}^{>\epsilon}$. In practice, one may further approximate the compound Poisson update by allowing at most one large jump within each time step. Let $\Lambda(\bm{x}_t,t)=\lambda(\bm{x}_t,t)\Delta t$. The probability of two or more large jumps in the same interval is
\begin{align}
\Pr\left(N_t^{>\epsilon}\geq 2\right)=&1-e^{-\Lambda(\bm{x}_t,t)}\left(1+\Lambda(\bm{x}_t,t)\right)\nonumber\\
\leq&\frac{\Lambda(\bm{x}_t,t)^2}{2}
\end{align}

Therefore, the single jump approximation introduces an additional error of order
\begin{align}
E_{\mathrm{multi}}=O\left(\sup_{t,\bm{x}_t}\lambda(\bm{x}_t,t)^2\Delta t^2\right),
\end{align}
provided that $\lambda(\bm{x}_t,t)\Delta t$ is uniformly small.

\vspace{-0cm}
\subsection{Approximate observation-guided sampler}
Posterior sampling for the diffusion process can be derived straightforwardly based on Bayes' rule, i.e., $\nabla\log p(\bm{x}_t|\bm{y})=\nabla\log p(\bm{y}|\bm{x}_t)+\nabla\log p(\bm{x}_t)$. The score function is approximated by the output of the neural network, and the likelihood is obtained by approximating $p(\bm{y}|\bm{x}_t)$ with $p(\bm{y}|E[\hat{\bm{x}}_0|\bm{x}_t])$ \cite{paper25}. 

However, a similar strategy is not directly applicable to the jump component because the reversed jump kernel involves a nonlocal density ratio. Exact posterior sampling for L\'evy-driven SDEs is therefore considerably more challenging. In principle, the large jump rate should be replaced by the posterior large jump rate, which requires likelihood-reweighted integration over the empirical mixture at each reverse step. This computation is rather expensive, especially because the integral cannot be simplified in general.

In this work, we therefore use the prior-trained rate network as an amortized proposal rate for large jump events in the posterior sampler. This approximation means that the prior rate network is not interpreted as the exact posterior rate. Instead, it provides a computationally efficient proposal mechanism for deciding when nonlocal transitions are proposed. The observation information is then incorporated into the large jump amplitude and direction through likelihood-reweighted empirical mixture weights. This design prioritizes the posterior correction of the jump destination, which has a direct impact on the nonlocal transition once a large jump is triggered, while avoiding the high cost of recomputing the exact posterior jump rate at every reverse step.

In the experiments, we will evaluate its effect by comparing the prior proposal rate with the likelihood-reweighted posterior rate approximation and by examining the sensitivity of the final estimation performance to the jump-rate choice. In the following proposition, we analyze how the observation modifies the empirical mixture weights used for large jump amplitude sampling.

\begin{proposition}\label{proposition_4}
\textit{Let $\bm{y}$ be the observations used for posterior sampling of the forward SDE \eqref{SDE}. Assume that the observation depends on the clean state $X_0$ only, i.e., $p_{\bm{Y}|\bm{X}_0,\bm{X}_t}(\bm{y}|\bm{x}_{0},\bm{x}_t)=p_{\bm{Y}|\bm{X}_0}(\bm{y}|\bm{x}_{0})$. Then, the weighted distribution $g_Q(\bm{x}|\bm{x}_t,t)$ used for generating large jump samples should be replaced by
\begin{equation}
g_{\text{posterior},Q}(\bm{x})=\sum_{j=1}^{N}\frac{1_{\{\bm{x}_{0,j}\}}(\bm{x})p_{\bm{Y}|\bm{X}_0}(\bm{y}|\bm{x}_{0,j})Q_{t,\epsilon,j}}{\sum_{j=1}^{N}p_{\bm{Y}|\bm{X}_0}(\bm{y}|\bm{x}_{0,j})Q_{t,\epsilon,j}}
\end{equation}}
\end{proposition}
\begin{IEEEproof}
The proof is relegated to Appendix \ref{appendix_3}.
\end{IEEEproof}

Note that the likelihood can usually be efficiently obtained from the transition function between $\bm{x}_0$ and $\bm{y}$. For instance, this transition function is the channel model for communication tasks. Proposition \ref{proposition_4} shows that the large jump amplitude sampler can be adapted to the observation-guided setting by modifying the empirical mixture weights. Similarly, $g_{\text{posterior},Q}(\bm{x})$ can also be truncated to reduce the inference latency.

\vspace{-0cm}
\section{Numerical results}\label{section_4}
In this section, several numerical experiments are provided to validate the advantages of the proposed inverse sampler. We consider channel estimation for OFDM-SISO systems under mixed noise.

\begin{figure*}[htbp]
\centering
\subfloat[]{\includegraphics[width=2.4in]{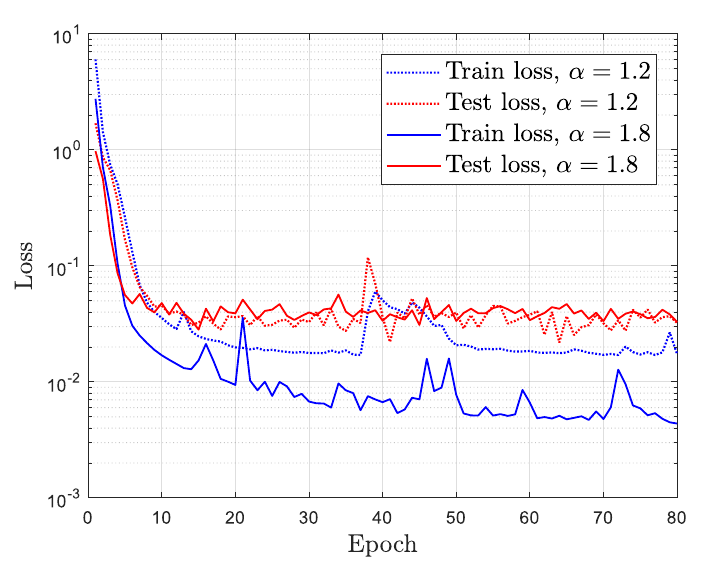}\label{fig_1_a}}
\subfloat[]{\includegraphics[width=2.4in]{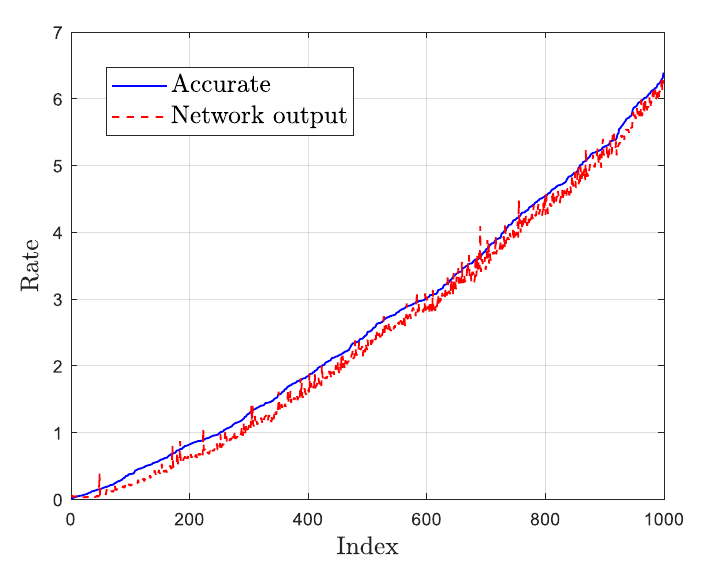}\label{fig_1_b}}
\subfloat[]{\includegraphics[width=2.4in]{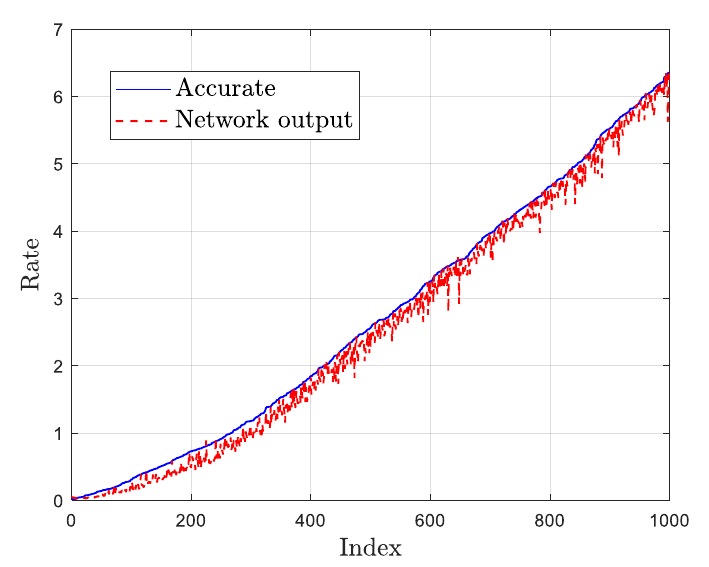}\label{fig_1_c}}
\caption{Training loss, test loss, and accuracy validation of jump rate prediction using the proposed lightweight neural network under different values of $\alpha$. (a) Comparison of training and test losses; (b) Comparison of jump rates for $\alpha=1.2$; (c) Comparison of jump rates for $\alpha=1.8$.}
\vspace{-0cm}
\label{fig_1}
\end{figure*}

For the channel model, we consider a frequency-selective SISO-OFDM channel modeled by a tapped-delay-line (TDL) response. The path gains are modeled as independent circularly symmetric complex Gaussian random variables with prescribed average powers. Moreover, the delay-power profiles from 3GPP TR 38.901 are used. Due to space limitations, we choose the TDL-A, TDL-C, and TDL-D delay-power profiles as representative scenarios. The mixed channel noise consists of WGN and IN described by the S$\alpha$S distribution. They are mutually independent and memoryless. Let their scaling parameters be $\gamma_g$ and $\gamma_s$, respectively. Then, the generalized signal-to-noise ratio (GSNR) is defined as follows:
\begin{equation}\label{GSNR}
\text{GSNR(dB)}=10\log_{10}\frac{P_s}{2(\gamma_g^2+\gamma_s^2)},
\end{equation}
where $P_s$ is the transmit signal power.

For the baselines, we consider existing model-based algorithms and generative methods, including:
\begin{itemize}
\item{\textbf{Linear minimum mean square error (LMMSE)}:
    LMMSE is an efficient channel estimation method under WGN. However, its performance and stability may deteriorate considerably in mixed-noise scenarios due to impulsive outliers.}
\item{\textbf{Clipped LMMSE}:
    Clipped LMMSE first suppresses large-amplitude pilot observations through clipping and then applies the conventional LMMSE estimator. This simple preprocessing improves the robustness of LMMSE against IN while retaining low computational complexity.}
\item{\textbf{Clipped orthogonal matching pursuit (OMP) \cite{paper27}}:
    OMP iteratively selects dominant delay atoms from a common candidate dictionary and reconstructs the frequency-domain channel. Clipped OMP is a clipped version of OMP for  scenarios with mixed noise.}
\item{\textbf{Clipped sparse Bayesian learning (SBL)}:
    SBL estimates the delay-domain sparse channel by learning the hyperparameters that reflect the sparsity. Clipped SBL is a clipped version of conventional SBL.}
\item{\textbf{Outlier-aware SBL \cite{paper26,paper28}}:
    Different from Clipped SBL, outlier-aware SBL explicitly models the received pilot observations as the sum of a sparse delay-domain channel component, a sparse outlier component, and Gaussian background noise. By jointly estimating the channel and the impulsive outliers, this method is designed for mixed Gaussian and IN environments.}
\item{\textbf{Diffusion method \cite{paper15}}:
    This method uses a generative model based on the diffusion forward process and posterior inverse inference. Different from the proposed framework, it does not introduce the jump process and uses only the local score function to recover the target distribution.}
\item{\textbf{Lévy-DSM \cite{paper19}}:
    This method is based on the inverse sampler designed for the Lévy process under the DSM framework. It shows that the network can be trained to regress noise following $\alpha$-stable distributions to achieve inverse sampling.}
\item{\textbf{Benchmark}:
    The benchmark is a genie-aided LMMSE estimator under WGN with known true delay support. It uses ideal prior delay information that is unavailable to practical estimators. Therefore, it serves as an NMSE reference for the achievable estimation performance.}
\end{itemize}

Finally, the other parameter configurations are as follows: the number of subcarriers is $N_c=64$; the number of OFDM symbols per frame is $N_t=16$; the pilot pattern is ``comb''; the pilot spacing is 4 or 8; the modulation scheme is QPSK; and the number of paths, path delays, and path powers follow TDL-A/C/D. The pilot spacing is set to either 4 or 8 to validate the performance under different pilot densities.

\subsection{Structure and training of network $g^{\theta_2}$}
Since the large jump rate is a scalar statistical quantity, a lightweight neural network is sufficient for its amortized estimation. Specifically, we parameterize $g_{\theta_2}(\bm{x}_t,t)$ by a convolutional neural network (CNN) and a multilayer perceptron (MLP). Its input is the normalized state $\bm{x}_t$ concatenated with a time embedding of $t$. To ensure the nonnegativity of the estimated jump rate, the final output is passed through a softplus activation. Compared with ReLU, softplus provides a smoother parameterization and avoids permanently inactive zero-rate outputs. In our implementation, the network is trained by minimizing the mean square error (MSE) between its output and $\lambda(\bm{x}_t|\bm{x}_0)$. Note that the original sample $\bm{x}_0$ is only used to construct the conditional training target. It is not fed into the network as input because an accurate $\bm{x}_0$ is not available at the inference stage.

For the OFDM channel estimation task, the complex-valued state is first represented by its real and imaginary parts. For example, when one frame contains $N_t$ OFDM symbols and $N_c$ subcarriers, we reshape the state into a two-channel tensor $\mathbf{X}_t\in\mathbb{R}^{2\times N_t\times N_c}$. The rate network adopts a small convolutional architecture to exploit the local time-frequency structure while keeping the number of trainable parameters low. Specifically, it consists of four $3\times 3$ convolutional layers with channel widths $\{16,16,32,32\}$, respectively, followed by a global average pooling layer. The pooled feature is concatenated with the time embedding and then passed to a two-layer MLP head to produce a scalar output. The estimated large jump rate is given by
\begin{align}
g^{\theta_2}(\mathbf{x}_t,t)=&\mathrm{softplus}(\mathrm{MLP}^{\theta_2}([\mathrm{GAP}(\mathrm{CNN}^{\theta_2}(\mathbf{X}_t)),\mathrm{Emb}(t)))\nonumber\\
&+\lambda_{\epsilon},
\end{align}
where $\mathrm{GAP}(\cdot)$ denotes global average pooling, $\mathrm{Emb}(t)$ is the time embedding, and $\lambda_{\epsilon}>0$ is a small constant for numerical stability. For $N_t=16$ and $N_c=64$, this small CNN contains approximately $2\times 10^4$ trainable parameters. This is much smaller than directly parameterizing the whole reverse transition with a neural network.

The MSE loss between the accurate rate and the network output is shown in Fig. \ref{fig_1}. Fig. \ref{fig_1_b} and Fig. \ref{fig_1_c} present the results of 1000 experiments, where the time instant and data sample $\bm{x}_{0}$ are chosen randomly. These 1000 tests are then sorted according to their accurate jump rates. For validation, the reference large jump rate is computed using the empirical marginal expression in \eqref{byhand_3}, rather than a single conditional target $\lambda(\bm{x}_t|\bm{x}_0)$. From Fig. \ref{fig_1_a}, it can be observed that the neural network converges very fast and only requires about 60 epochs with different values of $\alpha$. Moreover, the network can accurately estimate the jump rate under different impulsiveness levels.

\vspace{-0cm}
\begin{figure}[htbp]
\centering
\subfloat[NMSE, TDL-A, $\alpha=1.2$]{\includegraphics[width=1.75in]{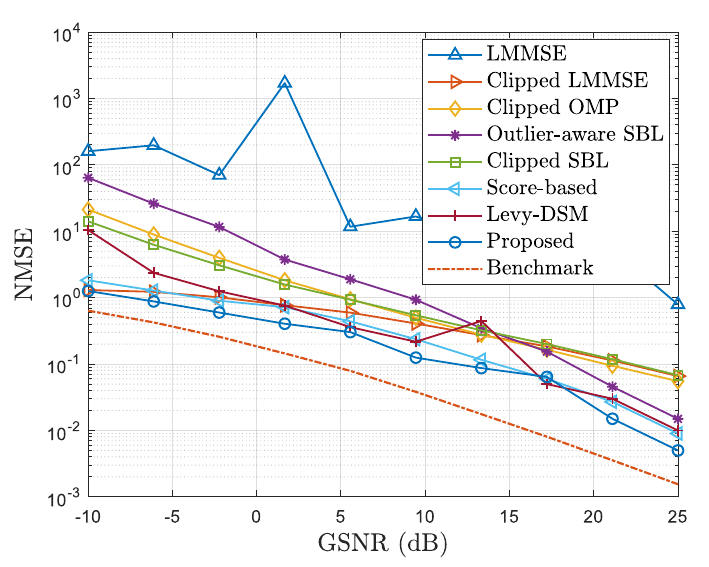}\label{fig_2_a}}
\subfloat[BER, TDL-A, $\alpha=1.2$]{\includegraphics[width=1.75in]{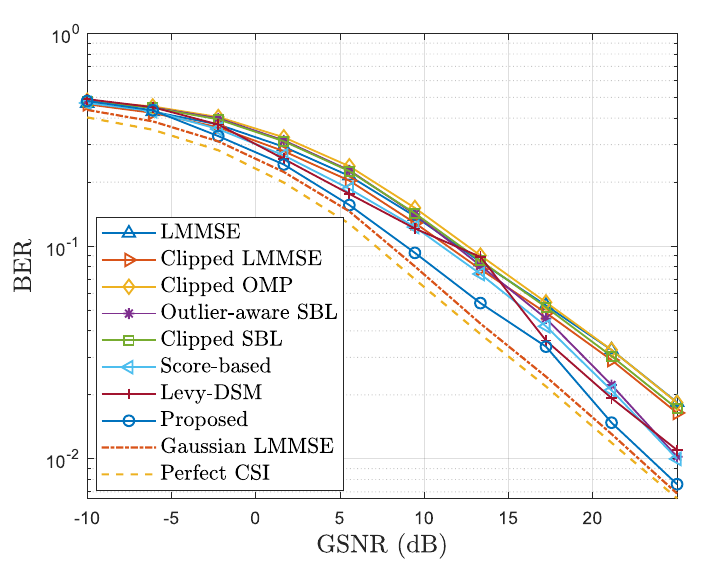}\label{fig_2_b}}

\subfloat[NMSE, TDL-C, $\alpha=1.2$]{\includegraphics[width=1.75in]{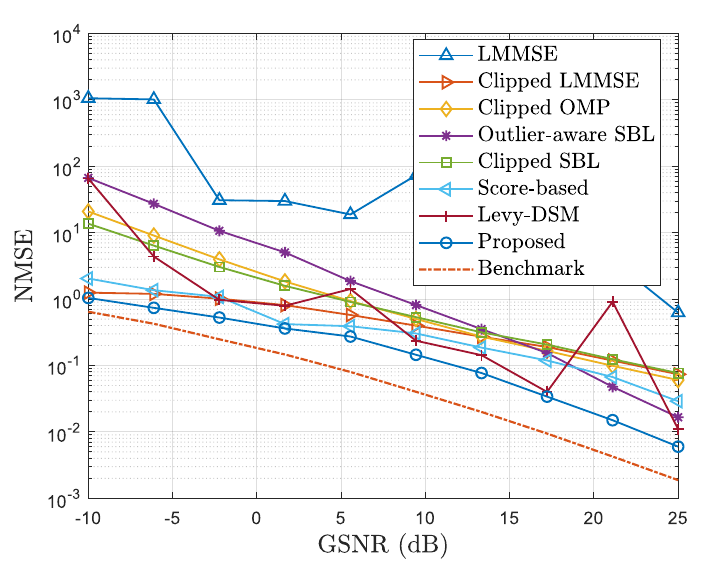}\label{fig_2_c}}
\subfloat[BER, TDL-C, $\alpha=1.2$]{\includegraphics[width=1.75in]{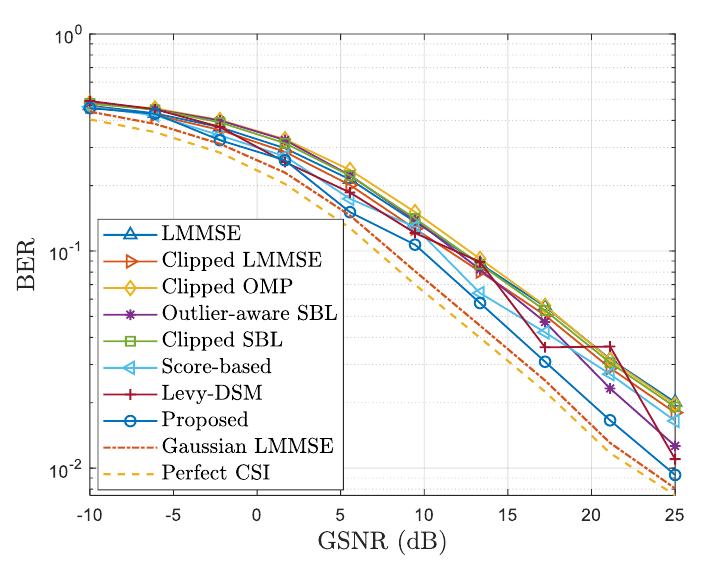}\label{fig_2_d}}

\subfloat[NMSE, TDL-D, $\alpha=1.2$]{\includegraphics[width=1.75in]{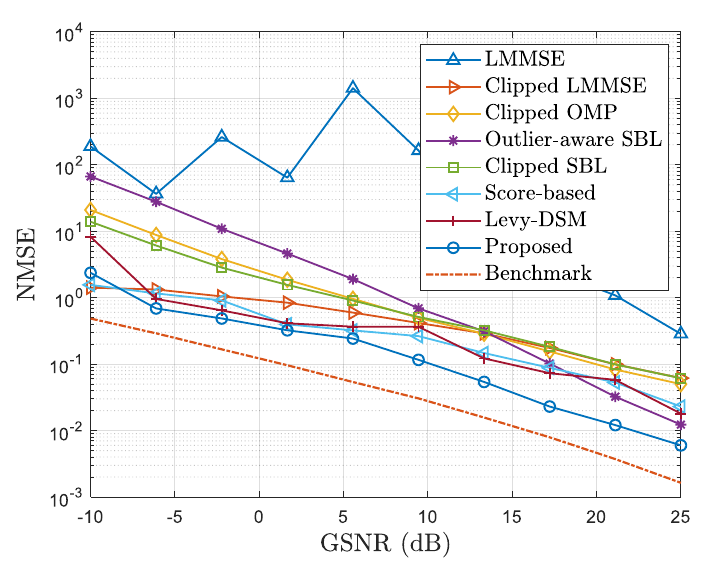}\label{fig_2_e}}
\subfloat[BER, TDL-D, $\alpha=1.2$]{\includegraphics[width=1.75in]{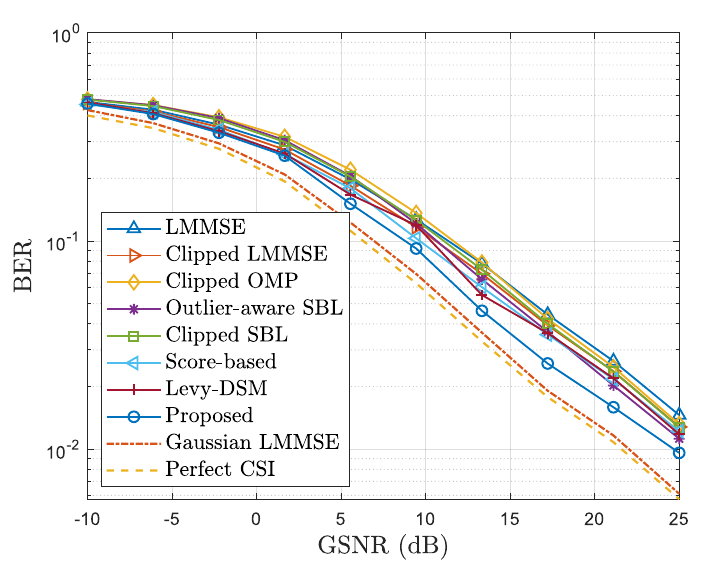}\label{fig_2_f}}
\caption{NMSE and BER comparisons of channel estimation under different scenarios (TDL-A/C/D) with $\alpha=1.2$.}
\label{fig_2}
\end{figure}
\vspace{-0cm}

\subsection{Various channel configurations}
Here, we investigate the channel estimation performance under various $\alpha$. First, the channel estimation accuracy and the corresponding detection performance are compared. We use normalized mean square error (NMSE) and bit error rate (BER) as the criteria. The experimental results are presented in Fig. \ref{fig_2} and Fig. \ref{fig_2_plus}.

In general, the proposed approach achieves the best performance among the baselines in terms of both NMSE and BER. For ``Clipped LMMSE'', ``Clipped SBL'', and ``Clipped OMP'', the clipping operation suppresses most IN samples and thus improves robustness. However, clipping also distorts the transmitted signals, especially in the high-GSNR regime. Therefore, these three methods show a lower NMSE decay rate as the GSNR continues to increase. ``Outlier-aware SBL'' has better robustness under mixed channel noise because it explicitly considers the influence of IN samples. However, it can only achieve an accurate estimation at the pilot locations. For learning-based methods, both ``Score-based'' and ``Proposed'' learn the statistical structure of the channel model and thus achieve better estimation accuracy in various scenarios. Compared with the ``Score-based'' algorithm, the proposed method is more likely to reach a better performance because it allows large jumps to enable nonlocal transitions between distant high-probability regions. In contrast, ``Score-based'' only uses gradient information and small step sizes for optimization. As a result, it may get trapped in local optima, especially when the channel model is more complicated. Finally, the proposed method is also more robust and stable compared to the ``Lévy-DSM'', which is consistent with the above analyses.

Fig. \ref{fig_2_plus} presents the estimation results under different impulsiveness levels. When $\alpha=1.2$, there are multiple outliers with large amplitudes, and the channel becomes much more complex. In this case, the proposed approach is not only robust but also achieves a larger performance gain compared with the cases with larger $\alpha$. However, when $\alpha$ is close to 2 and the impulsiveness is very weak, the existing baselines also show comparable performance. This is consistent with the above analysis.

\vspace{-0cm}
\begin{figure}[htbp]
\centering
\subfloat[NMSE, TDL-C, $\alpha=1.8$]{\includegraphics[width=1.75in]{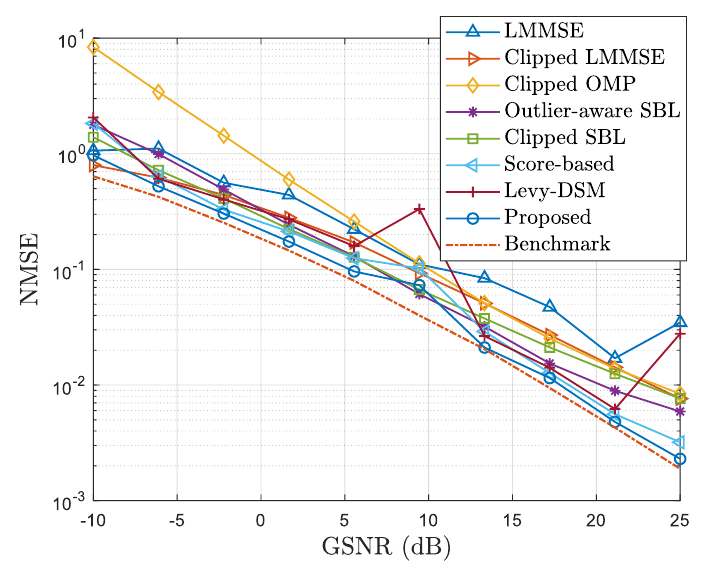}\label{fig_2_g}}
\subfloat[BER, TDL-C, $\alpha=1.8$]{\includegraphics[width=1.75in]{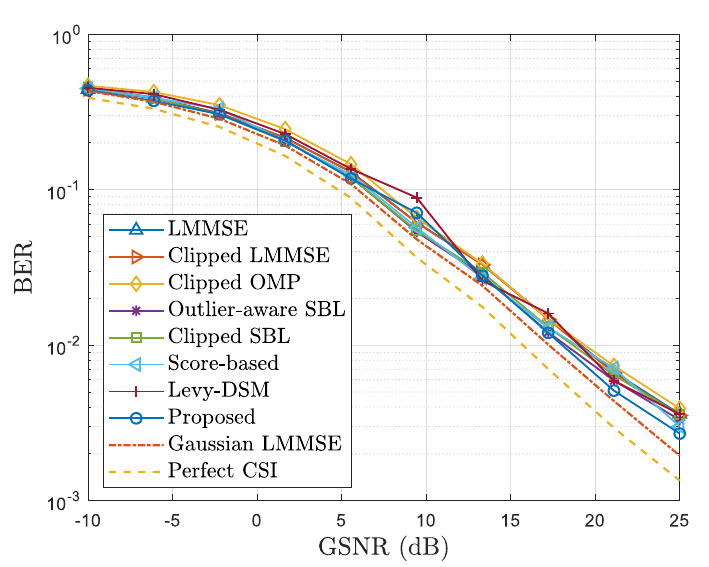}\label{fig_2_h}}
\caption{NMSE and BER comparisons of channel estimation with TDL-C as the delay-power profile and $\alpha=1.8$.}
\label{fig_2_plus}
\end{figure}
\vspace{-0cm}

\subsection{Ablations}
Ablation studies are conducted to explore the influence from the approximation of the forward noise distribution, the top-$K$ truncation and the approximated posterior jump rate.

During the design of the inverse sampler, the distribution of the forward mixed noise at arbitrary time instants is required. However, there is no exact closed-form PDF expression for the impulsive component or the mixed noise. To avoid numerical integration, we use the approximated PDF of the mixed noise  given in \eqref{approximation_of_X_M_1_dim} and \eqref{approxiamted_p_xt}. The following experiments demonstrate the effectiveness of this approximation. We still use channel estimation as the application. 

In Fig. \ref{fig_3}, Fig. \ref{fig_4} and Fig. \ref{fig_5}, we separately focus on the effects of noise distribution approximation, top-$K$ truncation and large jump rate approximation. The time consumption is shown in Table \ref{table_2}, using a CPU configuration of `12th Gen Intel(R) Core(TM) i9-12900H'. It can be observed that using the accurate PDF does not bring a considerable performance gain, while the complexity of ``Accurate PDF'' is unacceptable. This is because the distribution needs to be calculated by numerical integration at every step. More importantly, even though the high-dimensional PDF can be transformed into polar coordinates, the integral is still challenging to further simplify. This remains time-consuming when a high resolution is required to ensure accuracy.

Then, the approximation error introduced by the top-$K$ truncation is examined. For the benchmark, we set $K$ equal to the total number of samples used for training. Here, for the truncation scheme, we set $K$ as $10\%$ of the total number of training samples. Similar to the noise model approximation, the truncation error is also not significant. This is because, for a given time instant and sample $\bm{x}_t$, the probability that most original data samples generate $\bm{x}_t$ is quite small. Therefore, the influence of these samples on jump generation is negligible. As a result, $K$ can be much smaller than the total number of samples, which reduces redundancy and improves efficiency.

Finally, Fig. \ref{fig_5} implies that the posterior inverse sampling based on the exact large jump rate achieves a better NMSE. However, its computational cost significantly increases since the jump rate calculation needs to rely on the numerical integral. Moreover, in terms of detection error probability, the performance gain is less considerable than that with respect to NMSE. This result supports that using prior large jump rate is an efficient scheme. 

\textbf{In summary, the approximation operations in the inverse sampling algorithm achieve a tradeoff between performance and complexity.}

\vspace{-0cm}
\begin{figure}[htbp]
\centering
\subfloat[NMSE]{\includegraphics[width=1.75in]{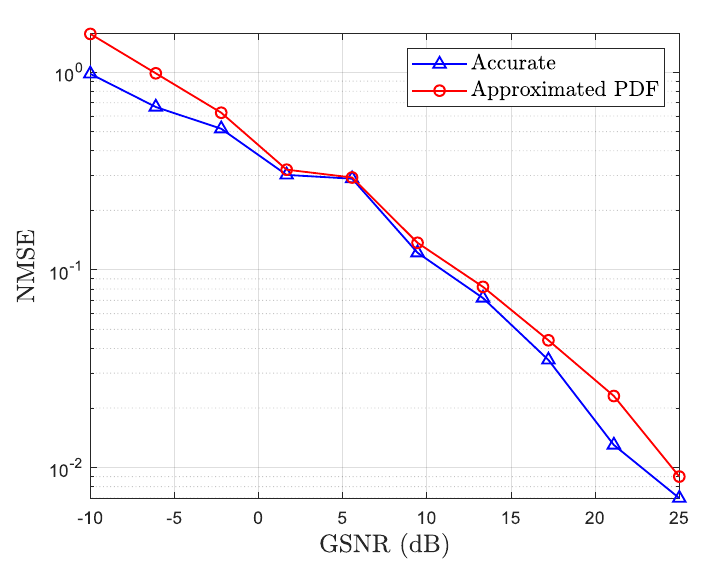}\label{fig_3_a}}
\subfloat[BER]{\includegraphics[width=1.75in]{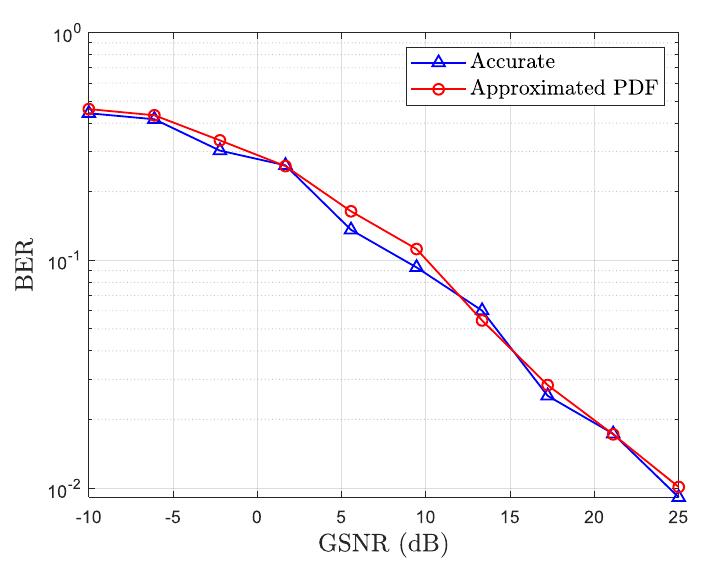}\label{fig_3_b}}
\caption{Ablation results compared with the accurate mixed noise PDF in terms of NMSE and BER.}
\label{fig_3}
\end{figure}
\vspace{-0cm}

\begin{figure}[htbp]
\centering
\subfloat[NMSE]{\includegraphics[width=1.75in]{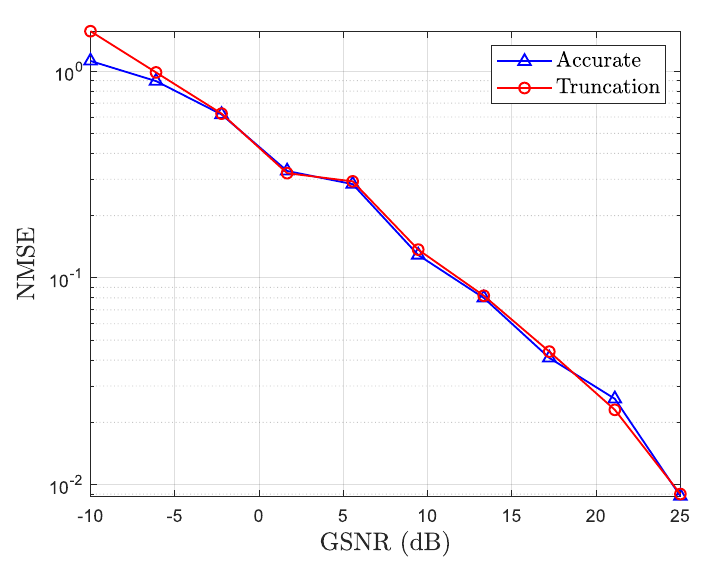}\label{fig_4_a}}
\subfloat[BER]{\includegraphics[width=1.75in]{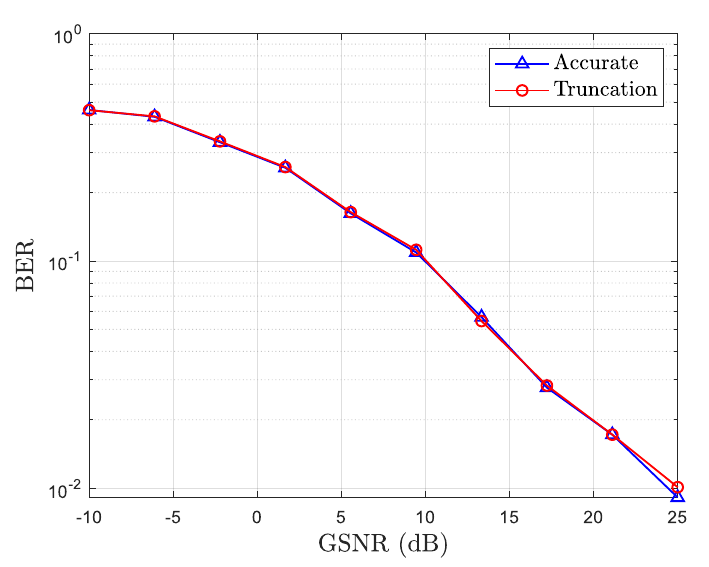}\label{fig_4_b}}
\caption{Ablation results compared with the no-truncation case in terms of NMSE and BER.}
\label{fig_4}
\end{figure}
\vspace{-0cm}

\begin{figure}[htbp]
\centering
\subfloat[NMSE]{\includegraphics[width=1.75in]{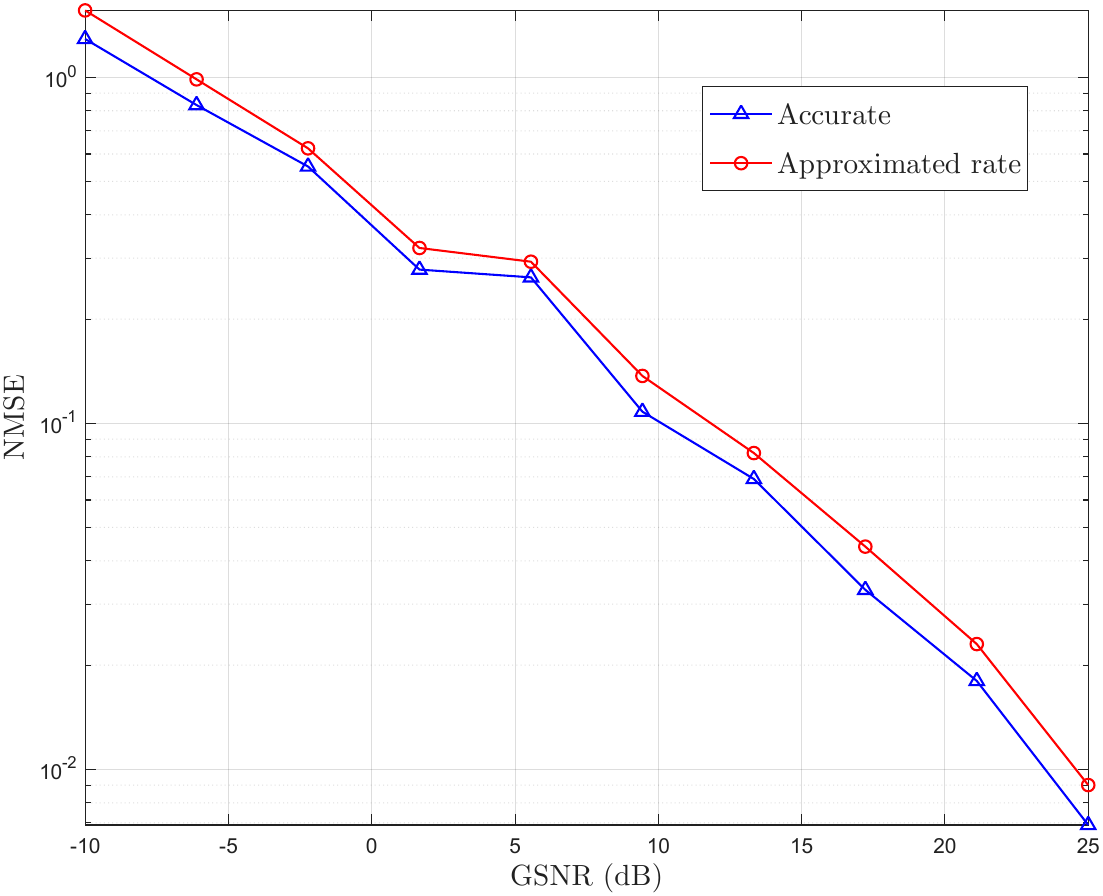}\label{fig_4_a}}
\subfloat[BER]{\includegraphics[width=1.75in]{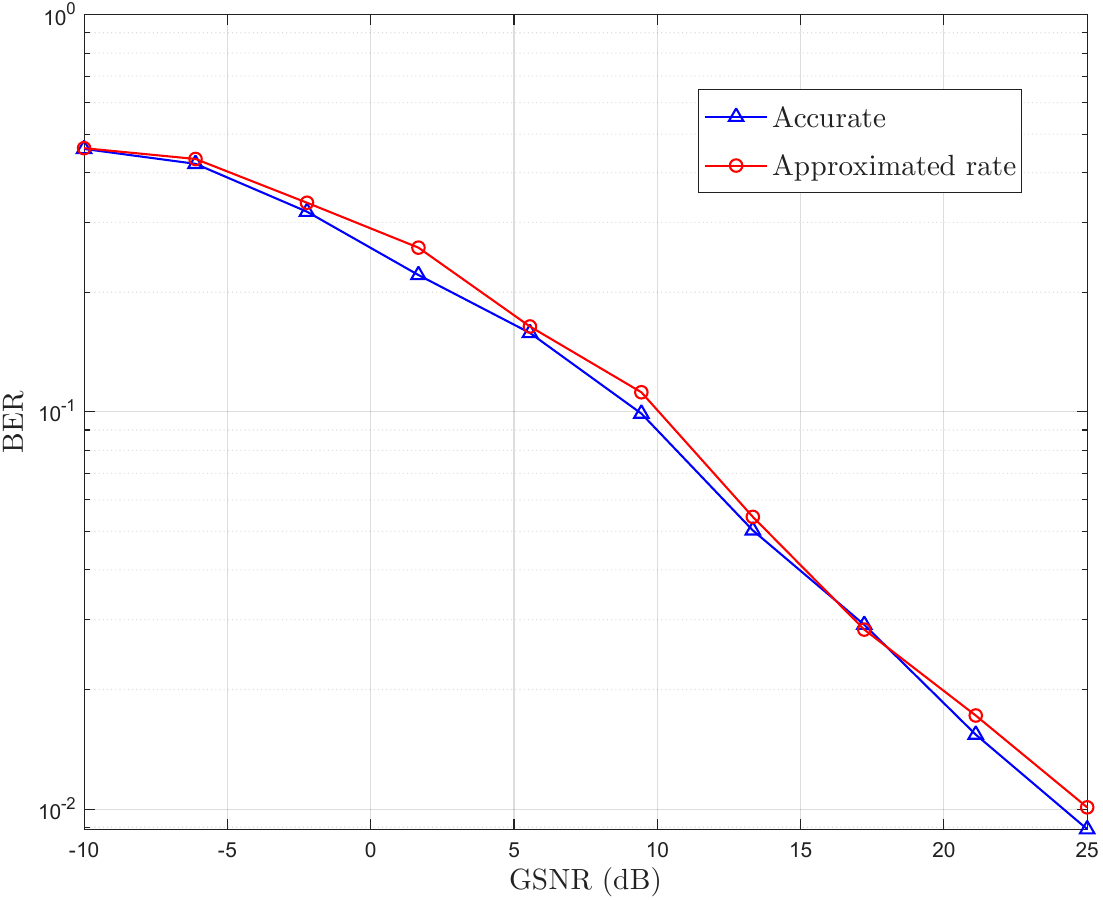}\label{fig_4_b}}
\caption{Ablation results compared with the accurate posterior large jump rate in terms of NMSE and BER.}
\label{fig_5}
\end{figure}

\linespread{1.2}
\begin{table}[htbp]
\begin{center}
\caption{Time consumption comparison of simulations in Fig. \ref{fig_3}$\sim$\ref{fig_5} under different values of $\alpha$.}\label{table_2}
\begin{tabular}{ccc}
\toprule
Algorithms/Scenarios& $\alpha=1.2$ & $\alpha=1.8$\\
\midrule
\textbf{Proposed}& \textbf{5.6s}& \textbf{5.5s}\\
Accurate PDF& 1717.6s& 1698.4s\\
No truncation& 52.7s& 54.1s\\
Accurate rate& 6125.3s& 6379.6s\\
\bottomrule
\vspace{-0cm}
\end{tabular}
\end{center}
\end{table}
\linespread{1}

\vspace{-0cm}
\section{Conclusions}\label{section_5}
In this paper, we proposed a generator-guided inverse sampling algorithm for a class of isotropic linear L\'evy-driven SDEs with symmetric $\alpha$-stable jump components. We first derived the expression of the backward generator and used it to decompose the full reverse sampling process into the diffusion part and the jump components. Then, by theoretically analyzing the inverse sampling process from a statistical perspective, we determined the objective that can be learned by a lightweight network. We also designed an efficient reverse sampling method by introducing several efficient approximation operations. As an application, the proposed inverse sampling algorithm outperformed the considered baselines in channel estimation under mixed channel noise across different scenarios. Moreover, numerical results showed that the proposed approach achieves a good balance between performance and complexity.

\begin{appendices}
\section{Proof of Proposition \ref{Proposition_1}}\label{appendix_1}
\setcounter{equation}{0}
\renewcommand{\theequation}{A.\arabic{equation}}
For a Lévy process with infinite activity, the generator associated with the test function $f$ is given by
\begin{align}\label{Lévy_forward_generator}
&(\mathcal{L}_{J,F}f)(\bm{x}_t)=\bm{b}_{J,F}^{\epsilon}(\bm{x}_t,t)\cdot\nabla f(\bm{x}_t) \nonumber\\
&+ \int_{R^D} {\left( {f(\bm{x}_t+\bm{v}) - f(\bm{x}_t) - 1_{\{\bm{v}:\Vert\bm{v}\Vert \le \epsilon\}}\bm{v}\cdot\nabla f(\bm{x}_t)} \right)\nu_t (d\bm{v})}
\end{align}
where $\bm{b}_{J,F}^{\epsilon}(\bm{x}_t,t)$ is the canonical drift induced by $d\bm{L}_t$. For the isotropic case, we have $\bm{b}_{J,F}^{\epsilon}(\bm{x}_t,t)=0$. By substituting \eqref{Lévy_forward_generator} into \eqref{forward_generator}, the complete forward generator is obtained as
\begin{align}\label{complete_forward_generator}
&\left( {\mathcal{L}_Ff} \right)(\bm{x}_t)=\bm{b}_F^{\epsilon}(\bm{x}_t,t)\cdot\nabla f(\bm{x}_t)+\frac{1}{2}{\rm{Tr}}\left( {\Sigma (t)\nabla^2 f(\bm{x}_t)} \right)\nonumber\\
&+\int_{R^D} {\left( {f(\bm{x}_t+\bm{v}) - f(\bm{x}_t) - 1_{\{\bm{v}:\Vert\bm{v}\Vert \le \epsilon\}}\bm{v}\cdot\nabla f(\bm{x}_t)} \right)\nu_t (d\bm{v})}
\end{align}
with $\bm{b}_F^{\epsilon}(\bm{x}_t,t)=\bm{b}(\bm{x}_t,t)+\bm{b}_{J,F}^{\epsilon}(\bm{x}_t,t)$. Next, time reversal is applied to derive the backward generator. Based on \cite{paper21}, the forward and backward jump kernels satisfy the flux equation
\begin{align}\label{flux_equation}
p(\bm{x}_t){{\mathord{\buildrel{\lower3pt\hbox{$\scriptscriptstyle\leftarrow$}}
\over K} }_{t,\bm{x}_t}}(d\bm{v})=p(\bm{x}_t+\bm{v}){{\mathord{\buildrel{\lower3pt\hbox{$\scriptscriptstyle\rightarrow$}}
\over K} }_{t,\bm{x}_t+\bm{v}}}(d(-\bm{v}))
\end{align}
which means that the jump density from $\bm{x}_t$ to $\bm{x}_t+\bm{v}$ should be equal to that from $\bm{x}_t+\bm{v}$ back to $\bm{x}_t$. Moreover, the backward canonical drift satisfies \cite{paper21}
\begin{align}\label{backward_drifting_Lévy_condition}
\bm{b}_{J,F}^{\epsilon}(\bm{x}_t,t)&+\bm{b}_{J,B}^{\epsilon}(\bm{x}_t,t)\nonumber\\
=&\int_{R^D} {1_{\{\bm{v}:\Vert\bm{v}\Vert\le\epsilon\}}\bm{v}\left( {{\mathord{\buildrel{\lower3pt\hbox{$\scriptscriptstyle\rightarrow$}}
\over K} }_{t,\bm{x}_t}(d\bm{v}) + {\mathord{\buildrel{\lower3pt\hbox{$\scriptscriptstyle\leftarrow$}}
\over K} }_{t,\bm{x}_t}(d\bm{v})} \right)}
\end{align}

Therefore, we have
\begin{align}\label{backward_kernel}
{{\mathord{\buildrel{\lower3pt\hbox{$\scriptscriptstyle\leftarrow$}}
\over K} }_{t,\bm{x}_t}}(d\bm{v}) =&\frac{p(\bm{x}_t+\bm{v})}{p(\bm{x}_t)}{{\vec K}_{t,\bm{x}_t+\bm{v}}}\left( {d\left( { - \bm{v}} \right)} \right)\nonumber\\
=&\frac{p(\bm{x}_t+\bm{v})}{p(\bm{x}_t)}\nu_t (d(-\bm{v}))
\end{align}

In this case, the backward generator of the Lévy process is given by
\begin{align}\label{backward_generator_Lévy}
&(\mathcal{L}_{J,B}f)(\bm{x}_t)\nonumber\\
=&\bm{b}_{J,B}^{\epsilon}(\bm{x}_t,t)\cdot\nabla f(\bm{x}_t) + \int_{R^D} \big( f(\bm{x}_t+\bm{v}) - f(\bm{x}_t)\nonumber\\
&- 1_{\{\bm{v}:\Vert\bm{v}\Vert\le\epsilon\}}\bm{v}\cdot\nabla f(\bm{x}_t) \big){{\mathord{\buildrel{\lower3pt\hbox{$\scriptscriptstyle\leftarrow$}}
\over K} }_{t,\bm{x}_t}}(d\bm{v})\nonumber\\
=&\bigg(-\bm{b}_{J,F}^{\epsilon}(\bm{x}_t,t)+\int_{R^D} 1_{\{\bm{v}:\Vert\bm{v}\Vert\le\epsilon\}}\bm{v}\cdot \Big(\nu_t(d\bm{v})\nonumber\\
&+{{\mathord{\buildrel{\lower3pt\hbox{$\scriptscriptstyle\leftarrow$}}
\over K} }_{t,\bm{x}_t}}(d\bm{v})\Big)\bigg)\cdot\nabla f(\bm{x}_t)+\int_{R^D}\big( f(\bm{x}_t+\bm{v}) - f(\bm{x}_t)\nonumber\\
&- 1_{\{\bm{v}:\Vert\bm{v}\Vert\le\epsilon\}}\bm{v}\cdot\nabla f(\bm{x}_t) \big)\mathord{\buildrel{\lower3pt\hbox{$\scriptscriptstyle\leftarrow$}}
\over K}_{t,\bm{x}_t}(d\bm{v})\nonumber\\
=&\bigg(-\bm{b}_{J,F}^{\epsilon}(\bm{x}_t,t)+\operatorname{p.v.}\int_{R^D} 1_{\{\bm{v}:\Vert\bm{v}\Vert\le\epsilon\}}\bm{v} \nu_t(d\bm{v})\bigg)\cdot\nabla f(\bm{x}_t)\nonumber\\
&+\operatorname{p.v.}\int_{R^D} \left(f(\bm{x}_t+\bm{v})-f(\bm{x}_t)\right)\frac{{{p}(\bm{x}_t+\bm{v})}}{{{p}(\bm{x}_t)}}\nu_t (d\bm{v})
\end{align}

Combining \eqref{complete_forward_generator} with the time reversal of the diffusion part, the complete backward generator is given by
\begin{align}\label{backward_generator_complete}
\left( {\mathcal{L}_Bf} \right)(\bm{x}_t) =&\big(-\bm{b}(\bm{x}_t,t)+\Sigma(t)\nabla\log p(\bm{x}_t)\big)\cdot\nabla f(\bm{x}_t)\nonumber\\
&+ \frac{1}{2}{\rm{Tr}}\left( {\Sigma(t) \nabla^2 f(\bm{x}_t)} \right) +(\mathcal{L}_{J,B}f)(\bm{x}_t)
\end{align}

Finally, if the scaled L\'evy measure $\nu_t$ is symmetric, we have $\nu_t (d(-\bm{v}))=\nu_t (d\bm{v})$. For the symmetric $\alpha$-stable case, the canonical jump drift vanishes. This further simplifies the backward generator and completes the proof.

\section{Proof of Proposition \ref{Proposition_3}}\label{appendix_2}
\setcounter{equation}{0}
\renewcommand{\theequation}{B.\arabic{equation}}
We first use a variable transformation to absorb the drift term. Let $\bm{Y}_t=\exp(-Rt)\bm{X}_t$. Based on Ito's formula,
\begin{align}
d\bm{Y}_t=&\exp(-Rt)d\bm{X}_t+d(\exp(-Rt))\bm{X}_t\nonumber\\
=&\exp(-Rt)d\bm{X}_t-R\exp(-Rt)\bm{X}_t\nonumber\\
=&\exp(-Rt)(\bm{s}dt+\Phi_G(t)d\bm{W}_t+\Phi_S(t)d\bm{L}_t)
\end{align}

According to $\bm{Y}_t=\exp(-Rt)\bm{X}_t$,
\begin{align}
\bm{X}_t=&\underbrace{\exp(Rt)\bm{X}_0+\exp(Rt)\int_{0}^{t}\exp(-Rl)\bm{s}dl}_{\bm{\mu}(t,\bm{x}_0)}\nonumber\\
&+\underbrace{\int_{0}^{t}\exp((t-l)R)\Phi_G(l)d\bm{W}_l}_{\bm{G}(t)}\nonumber\\
&+\underbrace{\int_{0}^{t}\exp((t-l)R)\Phi_S(l)d\bm{L}_l}_{\bm{S}(t)}
\end{align}

Then, based on the linearity of Gaussian RV and $\alpha$-stable RV, the proof is completed.

\section{Proof of Proposition \ref{proposition_4}}\label{appendix_3}
\setcounter{equation}{0}
\renewcommand{\theequation}{C.\arabic{equation}}
First, we have $\frac{p(\bm{x}_t+\bm{v}|\bm{y})}{p(\bm{x}_t|\bm{y})}\nu_t(d\bm{v})=\frac{p(\bm{x}_t+\bm{v},\bm{y})}{p(\bm{x}_t,\bm{y})}\nu_t(d\bm{v})$. Since the denominator is independent of $\bm{v}$, it can be omitted during sampling. For clarity, we add subscripts to indicate the arguments of the distributions. The target sampling distribution is given by
\begin{align}\label{derive_jump_posterior_sampling}
&\int_{\Vert\bm{v}\Vert>\epsilon}p_{\bm{X}_t,\bm{Y}}(\bm{x}_t+\bm{v},\bm{y})\nu_t(d\bm{v})\nonumber\\
\propto&\int_{\Vert\bm{v}\Vert>\epsilon}p_{\bm{Y}|\bm{X}_t}(\bm{y}|\bm{x}_t+\bm{v})p_{\bm{X}_t}(\bm{x}_t+\bm{v})\Vert\bm{v}\Vert^{-\alpha-D}d\bm{v}\nonumber\\
=&\int_{\Vert\bm{u}-\bm{x}_t\Vert>\epsilon} p_{\bm{Y}|\bm{X}_t}(\bm{y}|\bm{u})p_{\bm{X}_t}(\bm{u})\Vert\bm{u}-\bm{x}_t\Vert^{-\alpha-D}d\bm{u}\nonumber\\
=&\int_{\Vert\bm{u}-\bm{x}_t\Vert>\epsilon} \int_{\mathcal{R}^D}p_{\bm{Y}|\bm{X}_0}(\bm{y}|\bm{x}_0)p_{\bm{X}_0|\bm{X}_t}(\bm{x}_0|\bm{u})d\bm{x}_0 p_{\bm{X}_t}(\bm{u})\nonumber\\
&\times\Vert\bm{u}-\bm{x}_t\Vert^{-\alpha-D}d\bm{u}\nonumber\\
=&\int_{\Vert\bm{u}-\bm{x}_t\Vert>\epsilon} \int_{\mathcal{R}^D}p_{\bm{Y}|\bm{X}_0}(\bm{y}|\bm{x}_0)p_{\bm{X}_t|\bm{X}_0}(\bm{u}|\bm{x}_0) p_{\bm{X}_0}(\bm{x}_0)d\bm{x}_0\nonumber\\
&\times\Vert\bm{u}-\bm{x}_t\Vert^{-\alpha-D}d\bm{u}\nonumber\\
\approx&\int_{\Vert\bm{u}-\bm{x}_t\Vert>\epsilon} \frac{1}{N}\sum_{j=1}^{N} p_{\bm{Y}|\bm{X}_0}(\bm{y}|\bm{x}_{0,j})p_{\bm{X}_t|\bm{X}_0}(\bm{u}|\bm{x}_{0,j})\nonumber\\
&\times\Vert\bm{u}-\bm{x}_t\Vert^{-\alpha-D}d\bm{u}\nonumber\\
\propto&\int_{\Vert\bm{v}\Vert>\epsilon} \frac{1}{N}\sum_{j=1}^{N} p_{\bm{Y}|\bm{X}_0}(\bm{y}|\bm{x}_{0,j})p_{\bm{X}_t|\bm{X}_0}(\bm{v}+\bm{x}_t|\bm{x}_{0,j})\nu_t(d\bm{v})
\end{align}

By introducing the normalization factors and following procedures similar to those in \eqref{decompose_Q} and \eqref{rewritte_objective_PDF}, we can deduce that the target sampling distribution can also be decomposed into the conditional distribution and the weighted distribution $g_{\text{posterior},Q}(\bm{x})$, where
\begin{equation}
g_{\text{posterior},Q}(\bm{x})=\sum_{j=1}^{N}\frac{1_{\bm{x}_{0,j}}(\bm{x})p_{\bm{Y}|\bm{X}_0}(\bm{y}|\bm{x}_{0,j})Q_{t,\epsilon,j}}{\sum_{j=1}^{N}p_{\bm{Y}|\bm{X}_0}(\bm{y}|\bm{x}_{0,j})Q_{t,\epsilon,j}}
\end{equation}
\end{appendices}

\footnotesize
\bibliographystyle{IEEEtran}
\bibliography{ref}

@book{book1,
address={Cambridge},
series={Cambridge Studies in Advanced Mathematics},
title={Lévy Processes and Stochastic Calculus},
publisher={Cambridge University Press},
author={Applebaum, David},
year={2004}}

@book{book2,
editor = {Daniel Zwillinger and Victor Moll and I.S. Gradshteyn and I.M. Ryzhik},
title = {Table of Integrals, Series, and Products (Eighth Edition)},
publisher = {Academic Press},
address = {Boston},
year = {2014}
}

@book{book3,
title={Stable non-Gaussian random processes: Stochastic models with infinite variance},
author={Gennady Samorodnitsky and Murad S. Taqqu Chapman and Hall},
publisher={Chapman \& Hall},
address={New York},
year={1994}}

@article{paper1,
  title={Generative modeling by estimating gradients of the data distribution},
  author={Song, Yang and Ermon, Stefano},
  journal={Advances in neural information processing systems},
  volume={32},
  year={2019}
}

@article{paper2,
  title={Denoising diffusion probabilistic models},
  author={Ho, Jonathan and Jain, Ajay and Abbeel, Pieter},
  journal={Advances in neural information processing systems},
  volume={33},
  pages={6840--6851},
  year={2020}
}

@article{paper3,
      title={Score-Based Generative Modeling through Stochastic Differential Equations},
      author={Yang Song and Jascha Sohl-Dickstein and Diederik P. Kingma and Abhishek Kumar and Stefano Ermon and Ben Poole},
      year={2021},
      journal={International Conference on Learning Representations},
      volume={9},
      url= {https://mlanthology.org/iclr/2021/song2021iclr-scorebased/}
}

@ARTICLE{paper4,
  author={Vincent, Pascal},
  journal={Neural Computation},
  title={A Connection Between Score Matching and Denoising Autoencoders},
  year={2011},
  volume={23},
  number={7},
  pages={1661-1674},
  month={Jul.}}

@article{paper5,
author = {Li, Ji and Wang, Chao},
title = {Efficient Diffusion Posterior Sampling for Noisy Inverse Problems},
journal = {SIAM Journal on Imaging Sciences},
volume = {18},
number = {2},
pages = {1468-1492},
year = {2025}}

@article{paper6,
  title={Denoising diffusion restoration models},
  author={Kawar, Bahjat and Elad, Michael and Ermon, Stefano and Song, Jiaming},
  journal={Advances in neural information processing systems},
  volume={35},
  pages={23593--23606},
  year={2022}
}

@article{paper7,
  title={Improving diffusion models for inverse problems using manifold constraints},
  author={Chung, Hyungjin and Sim, Byeongsu and Ryu, Dohoon and Ye, Jong Chul},
  journal={Advances in Neural Information Processing Systems},
  volume={35},
  pages={25683--25696},
  year={2022}
}

@inproceedings{paper8,
  title={Solving inverse problems with latent diffusion models via hard data consistency},
  author={Song, Bowen and Kwon, Soo Min and Zhang, Zecheng and Hu, Xinyu and Qu, Qing and Shen, Liyue},
  booktitle={International Conference on Learning Representations},
  volume={2024},
  pages={7624--7654},
  year={2024}
}

@ARTICLE{paper9,
  author={Zilberstein, Nicolas and Sabharwal, Ashutosh and Segarra, Santiago},
  journal={IEEE Transactions on Signal Processing},
  title={Solving Linear Inverse Problems Using Higher-Order Annealed Langevin Diffusion},
  year={2024},
  volume={72},
  number={},
  pages={492-505},
  month={Jan.}}

@INPROCEEDINGS{paper10,
  author={Hu, Yuanyun and Bell, Evan and Wang, Guijin and Sun, Yu},
  booktitle={ICASSP 2026 - 2026 IEEE International Conference on Acoustics, Speech and Signal Processing (ICASSP)},
  title={PRISM: Probabilistic and Robust Inverse Solver with Measurement-Conditioned Diffusion Prior for Blind Inverse Problems},
  year={2026},
  volume={},
  number={},
  pages={11432-11436}}

@inproceedings{paper11,
  title={Diffir: Efficient diffusion model for image restoration},
  author={Xia, Bin and Zhang, Yulun and Wang, Shiyin and Wang, Yitong and Wu, Xinglong and Tian, Yapeng and Yang, Wenming and Van Gool, Luc},
  booktitle={Proceedings of the IEEE/CVF international conference on computer vision},
  pages={13095--13105},
  year={2023}
}

@INPROCEEDINGS{paper12,
  author={Fei, Ben and Lyu, Zhaoyang and Pan, Liang and Zhang, Junzhe and Yang, Weidong and Luo, Tianyue and Zhang, Bo and Dai, Bo},
  booktitle={2023 IEEE/CVF Conference on Computer Vision and Pattern Recognition (CVPR)},
  title={Generative Diffusion Prior for Unified Image Restoration and Enhancement},
  year={2023},
  volume={},
  number={},
  pages={9935-9946}}

@ARTICLE{paper13,
  author={Chen, Yunjin and Pock, Thomas},
  journal={IEEE Transactions on Pattern Analysis and Machine Intelligence},
  title={Trainable Nonlinear Reaction Diffusion: A Flexible Framework for Fast and Effective Image Restoration},
  year={2017},
  volume={39},
  number={6},
  pages={1256-1272},
  month={Jun.}}

@ARTICLE{paper14,
  author={Özdenizci, Ozan and Legenstein, Robert},
  journal={IEEE Transactions on Pattern Analysis and Machine Intelligence},
  title={Restoring Vision in Adverse Weather Conditions With Patch-Based Denoising Diffusion Models},
  year={2023},
  volume={45},
  number={8},
  pages={10346-10357},
  month={Aug.}}

@ARTICLE{paper15,
  author={Zhou, Xingyu and Liang, Le and Zhang, Jing and Jiang, Peiwen and Li, Yong and Jin, Shi},
  journal={IEEE Transactions on Wireless Communications},
  title={Generative Diffusion Models for High Dimensional Channel Estimation},
  year={2025},
  volume={24},
  number={7},
  pages={5840-5854},
  month={Jul.}}

@ARTICLE{paper16,
  author={Arvinte, Marius and Tamir, Jonathan I.},
  journal={IEEE Transactions on Wireless Communications},
  title={MIMO Channel Estimation Using Score-Based Generative Models},
  year={2023},
  volume={22},
  number={6},
  pages={3698-3713},
  month={Jun.}}

@ARTICLE{paper17,
  author={Li, Runhua and Sun, Jian and Xue, Jiang},
  journal={IEEE Journal on Selected Areas in Communications},
  title={Generative Diffusion-Based Bayesian Modeling for Universal Channel Estimation},
  year={2026},
  volume={44},
  number={},
  pages={3104-3119},
  month={Dec.}}

@article{paper18,
  title={Elucidating the design space of diffusion-based generative models},
  author={Karras, Tero and Aittala, Miika and Aila, Timo and Laine, Samuli},
  journal={Advances in neural information processing systems},
  volume={35},
  pages={26565--26577},
  year={2022}
}

@article{paper19,
  title={Score-based generative models with L{\'e}vy processes},
  author={Yoon, Eun Bi and Park, Keehun and Kim, Sungwoong and Lim, Sungbin},
  journal={Advances in neural information processing systems},
  volume={36},
  pages={40694--40707},
  year={2023}
}

@inproceedings{paper20,
  title={HEAVY-TAILED DIFFUSION WITH DENOISING L{\'e}vy PROBABILISTIC MODELS},
  author={Shariatian, Dario and Simsekli, Umut and Oliviero Durmus, Alain},
  booktitle={International Conference on Learning Representations},
  volume={2025},
  pages={96991--97024},
  year={2025}
}

@article{paper21,
title = {Time reversal of Markov processes with jumps under a finite entropy condition},
journal = {Stochastic Processes and their Applications},
volume = {144},
pages = {85-124},
year = {2022},
author = {Giovanni Conforti and Christian Léonard},
month={Feb.}}

@inproceedings{paper22,
  title     = {Flow Matching for Generative Modeling},
  author    = {Lipman, Yaron and Chen, Ricky T. Q. and Ben-Hamu, Heli and Nickel, Maximilian and Le, Matthew},
  booktitle = {The Eleventh International Conference on Learning Representations},
  year      = {2023},
  address   = {Kigali, Rwanda}
}

@ARTICLE{paper23,
  author={Sureka, Gaurav and Kiasaleh, Kamran},
  journal={IEEE Trans. Commun.},
  title={Sub-Optimum Receiver Architecture for AWGN Channel with Symmetric Alpha-Stable Interference},
  year={2013},
  month={May.},
  volume={61},
  number={5},
  pages={1926-1935}}

@ARTICLE{paper24,
  author={Qi, Tianfu and Zhang, Jing and Wang, Jun and Zhu, Yichao},
  journal={IEEE Trans. Commun.},
  title={Bursty Mixed Gaussian-Impulsive Noise Model and Parameter Estimation},
  year={2025},
  volume={73},
  number={9},
  pages={8274-8288},
  month={Sept.}}

@inproceedings{paper25,
  title     = {Diffusion Posterior Sampling for General Noisy Inverse Problems},
  author    = {Chung, Hyungjin and Kim, Jeongsol and McCann, Michael Thompson and Klasky, Marc Louis and Ye, Jong Chul},
  booktitle = {The Eleventh International Conference on Learning Representations (ICLR 2023)},
  year      = {2023},
  address   = {Kigali, Rwanda}
}

@ARTICLE{paper26,
  author={Feng, Xiao and Wang, Junfeng and Kuai, Xiaoyan and Zhou, Mingzhang and Sun, Haixin and Li, Jianghui},
  journal={IEEE Transactions on Vehicular Technology},
  title={Message Passing-Based Impulsive Noise Mitigation and Channel Estimation for Underwater Acoustic OFDM Communications},
  year={2022},
  volume={71},
  number={1},
  pages={611-625},
  month={Jan.}}

@ARTICLE{paper27,
  author={Wang, Zhizhan and Li, Yuzhou and Wang, Chengcai and Ouyang, Donghong and Huang, Yunlong},
  journal={IEEE Wireless Communications Letters},
  title={A-OMP: An Adaptive OMP Algorithm for Underwater Acoustic OFDM Channel Estimation},
  year={2021},
  volume={10},
  number={8},
  pages={1761-1765},
  month={Aug.}}

@ARTICLE{paper28,
  author={Zhang, Zheming and Han, Xiao and Li, Weizhe and Wei, Li and Yin, Jingwei},
  journal={IEEE Wireless Communications Letters},
  title={Robust Time-Correlated Sparse Bayesian Channel Estimation for Short-Block Underwater Acoustic Communications Under Impulsive Noise},
  year={2026},
  volume={15},
  number={},
  pages={3194-3198},
  month={May.}}

\end{document}